\documentclass{article}
\usepackage{amsmath,amssymb,fullpage, amsthm}
\usepackage{authblk}
\usepackage{hyperref}
\usepackage[capitalize,noabbrev]{cleveref}
\usepackage{graphicx, color}
\usepackage[rgb,dvipsnames]{xcolor}
\usepackage{wrapfig}
\usepackage{mathtools}
\usepackage{enumerate,enumitem}
\usepackage{bbm}
\usepackage[margin=1in]{geometry}
\usepackage{amssymb}
\usepackage{subcaption}
\usepackage{float}
\usepackage{tikz}
\usepackage{yhmath}
\usepackage[numbers]{natbib}
\usepackage{tabu}
\usepackage{tabularx}
\usepackage{algorithm}
\usepackage{algorithmic}
\usepackage{multirow}
\usepackage{makecell}

\usepackage{mathabx}

\usepackage[flushleft]{threeparttable}
\usepackage{array,booktabs,makecell}
\usepackage{microtype}
\usepackage{nicematrix}
\usepackage{arydshln}
\newtheorem{theorem}{Theorem}[section]
\newtheorem{proposition}[theorem]{Proposition}
\newtheorem{lemma}[theorem]{Lemma}
\newtheorem{corollary}[theorem]{Corollary}
\theoremstyle{definition}
\newtheorem{definition}[theorem]{Definition}

\theoremstyle{theorem}
\newtheorem{remark}[theorem]{Remark}

\newcommand{\R}{\mathbb{R}}

\newcommand{\A}{\mathcal{A}}

\newcommand{\RSnew}[1]{\textcolor{Black}{#1}}
\newcommand{\SZnew}[1]{\textcolor{Black}{#1}}

\title{
GPTQ-intrinsic LoRA: A Near-optimal Algorithm for Low-precision Quantization with Low-rank Adaptation
}

\begin{document}
\author[1]{Shihao Zhang}
\author[2]{Rayan Saab}

\affil[1]{Department of Mathematics, University of California San Diego\\
\texttt{shz051@ucsd.edu}}

\affil[2]{Department of Mathematics and Hal{\i}c{\i}o\u{g}lu Data Science Institute, University of California San Diego\\
\texttt{rsaab@ucsd.edu}}

\date{}
\maketitle

\begin{abstract}
Post-training quantization is widely used for compressing large neural networks, but aggressive low-bit quantization can significantly degrade model quality. A common remedy is to augment the quantized weights with a low-rank correction, leading to approximations of the form $W\approx Q+LR$. In this paper, we study this low-precision plus low-rank representation through the layer-wise reconstruction objective $\|XW-X(Q+LR)\|_F^2$, where $X$ is a calibration matrix. We establish, to our knowledge, the first information-theoretic lower bounds for this problem under finite-alphabet and bounded low-rank compensation constraints. We then propose GPTQ-intrinsic LoRA, a training-free algorithm that incorporates the low-rank correction directly into a GPTQ-style (\citet{frantar2022gptq})  quantization pass by appropriately augmenting the calibration Hessian. For the choice $L=V_r$, where $V_r$ contains the top right singular vectors of $X$, we prove layer-wise reconstruction error bounds in which the usual GPTQ dependence on $\|X\|_F^2$ is replaced by the rank-$r$ residual $\|X-X_r\|_F^2$, up to regularization terms. Under natural structural assumptions, these bounds match the information-theoretic lower bounds in their dominant scaling, up to constants and mild factors. We also introduce Bid-Up, a fixed-grid quantization refinement step that can be alternated with optimal low-rank compensation with guaranteed non-increasing layer-wise reconstruction error. Experiments on Qwen3 language models and DeiT vision transformers show that GPTQ-intrinsic LoRA improves over GPTQ and GPTQ followed by low-rank compensation, with additional gains from refinement loops.
\end{abstract}

\section{Introduction}\label{sec:intro}
Large neural networks, including transformer based large language models (LLMs) are revolutionizing a wide range of applications like natural language processing, content generation, robotics as well as vision and reasoning tasks. However, such models require substantial memory to store their weights and considerable computational resources for inference. Consequently, the demand for post training model size reduction and inference optimization has grown \cite{neill2020overview, deng2020model, cheng2017survey}, especially in the need of deploying LLMs on resource-constrained devices. Common approaches for compressing deep neural networks include low-rank approximation \cite{zhang2015accelerating, chen2021drone, zhang2025theoretical}, pruning \cite{child2019generating, hoefler2021sparsity, frantar2023sparsegpt} and quantization \cite{gholami2022survey, liang2021pruning, wei2024advances}. 

When used alone, low-precision quantization has been established as a mainstream compression technique  for its superior performance \cite{kuzmin2023pruning, yang2024llmcbench}. By reducing the number of bits used to represent weights (and activations), quantization lowers storage requirements, memory bandwidth, and computation costs. Post-training quantization (PTQ) \cite{lybrand2021greedy, frantar2022gptq, zhang2023post, zhang2026beacon} is a particularly attractive approach because of its simplicity. It avoids backpropagation and adapts pre-trained models usually in a single pass or in very few passes over a small calibration data set. However, pushing PTQ into sub-4-bit regimes with minimal model degradation is still challenging due to limited representation capacity. Thus, a combination of the three compression techniques mentioned above becomes a natural remedy. At a high level, this leads to a matrix decomposition viewpoint, in which neural network weight matrices are approximated by structured representations that incorporate low-precision, low-rank, or sparse components. More specifically, given a weight matrix $W$, one seeks an approximation $\widetilde{W}$ of one of the forms $\widetilde{W}=LR+S$, $\widetilde{W}=Q+S$, or $\widetilde{W}=Q+LR$, corresponding respectively to low-rank plus sparse, low-precision plus sparse, and low-precision plus low-rank decompositions.

Low-rank plus sparse matrix decomposition is the basis of robust principal component analysis \cite{candes2011robust}. It has been rigorously studied in the context of compressed sensing \cite{zhou2011godec, tanner2023compressed} and discrete optimization \cite{bertsimas2023sparse}, and it has been applied to the compression of deep neural networks \cite{yu2017compressing, zhang2024oats, zhang2023loraprune}. Similarly, the combination of low-precision quantization and sparsity has also been used for compressing LLMs \cite{dettmers2023spqr, kim2023squeezellm}.

In this paper, we focus on the third approach, namely low-precision plus low-rank decomposition. In practice, this often arises by applying low-rank adaptation (LoRA) \cite{hu2021lora} to a quantized model \cite{dettmers2024qlora}. Specifically, LoRA fine-tuning keeps the matrix of ``base" weights fixed, and adaptation is achieved by learning a low-rank update $LR$ that is added to it. In our setting, the base weights are quantized. This parameterization greatly reduces the number of trainable parameters. This $Q+LR$ approach has been empirically shown to recover much of the expressive power of the original full-precision pretrained model, as we will see when we review the relevant literature in \cref{sec:background}. Nevertheless, despite substantial work on algorithmic and systems aspects of combining low-precision quantization with low-rank adaptation, a theoretical treatment remains lacking. The empirical observation that a LoRA-finetuned quantized model can achieve performance close to that of its full-precision counterpart naturally leads to two central questions:
\begin{quote}
\itshape
How well can a matrix be approximated by the sum of a low-rank matrix and a low-precision matrix, and can such an approximation be computed efficiently?
\end{quote}

The first question concerns information-theoretic lower bounds, while the second concerns the design of  efficient and scalable algorithms with guarantees on the approximation error. In this paper, we answer both questions affirmatively.

{\bf Notation.} Throughout the paper, the weight matrix of a layer is denoted by $W \in\mathbb{R}^{N \times N'}$, where each of the $N'$ columns represents a $N$-dimensional channel. $\mathcal{A}$ denotes the discrete quantization grid (or alphabet) used for weight quantization. In a low-precision plus low-rank decomposition, $Q\in\A^{N \times N'}$ will be reserved for the low-precision component and $L\in\mathbb{R}^{N \times r},\,R\in\mathbb{R}^{r \times N'}$ for the left and right low-rank factors.

For a grid $\A$, we associate
a memoryless scalar quantizer (MSQ) $\mathcal{Q}:\mathbb{R}\rightarrow \mathcal{A}$ given by
$
\mathcal{Q}(z):=\arg\min_{p \in \mathcal{A}}|z-p|,
$ which essentially executes a ``round to nearest" (RTN) operation. Given a vector $v \in \mathbb{R}^n$, we use $v_i$ for its $i$-th entry, $v_{\geq j}$ for the subvector $(v_j, \dots, v_n)^\top$, and we define $v_{\leq j}$ analogously. $\|v\|$ is the Euclidean norm of $v$. Given a matrix $A \in \mathbb{R}^{m \times n}$, we use $A_i$ to denote its $i$-th column. We use $A_{\geq j}$ to denote the submatrix $(A_j, \dots, A_n)$. We use $\text{col}(A)$ to denote the column space of $A$. $P_A$ is the orthogonal projection onto $\text{col}(A)$, and $P_{A^\perp}$ the projection onto its orthogonal complement. The $\|\cdot\|_\infty$ norm on a matrix  refers to the entry-wise infinity norm.
Throughout this paper, all indices start from 1, following the standard mathematical convention.

\SZnew{In the experimental section of this paper, we focus on asymmetric per-channel PTQ. Let $\{0,1,\dots,2^b\!-\!1\}$ represent the $b$-bit uniform integer grid. The standard asymmetric weight quantizer uses its scaled and shifted version
\(
\mathcal{A} = \{ s \cdot (k-z) : k=0,\dots,2^b\!-\!1 \},
\)
where $s$ is the scaling factor and $z$ is the offset (zero point) defined per channel (column) $w \in \mathbb{R}^N$ of $W$. In asymmetric quantization, $s$ is typically defined on a scaled min–max grid by
\(
s = \frac{\beta (\max(w) - \min(w))}{2^b-1},
\) and we follow the standard choice of zero point $z=\left\lceil\frac{-\min(w)}{\max(w)-\min(w)} \cdot (2^b-1)\right\rfloor$. 
The associated RTN operator $\mathcal{Q}$ in this case is
\begin{equation*}
\mathcal{Q}(\star) = s \cdot \left( \text{clip}\!\left( \left\lceil \tfrac{\star}{s} + z \right\rfloor ; 0, 2^b-1 \right) - z \right),
\end{equation*}
where $\text{clip}(x;a_{\min},a_{\max}) = \min\{\max\{x,a_{\min}\},a_{\max}\}$. }

\textbf{Layer-wise reconstruction.}
When compressing a layer $W$ into $\widetilde{W}$, the goal is to preserve the input–output behavior of the neural network. Consequently, it is not necessary that $\widetilde{W}$ be close to $W$ in a matrix norm. On the other hand, optimizing $\widetilde{W}$ with respect to the end-to-end training loss typically requires additional retraining of the model, which can be prohibitively expensive for large-scale neural networks. An alternative is to perform block-wise reconstruction, where multiple layers are optimized jointly to match intermediate activations \cite{hubara2021accurate, li2021brecq}. However, such approaches often scale poorly with the size of modern large language models due to the resulting dimensionality of the optimization problem. As a result, many post-training compression methods rely on layer-wise reconstruction, where each layer is compressed sequentially, one at a time, by minimizing the discrepancy between the outputs of the original and compressed layer while keeping the rest of the network fixed. {The post training quantization algorithm GPTQ \cite{frantar2022gptq}, which we will mention briefly in \cref{sec:background} and introduce with more details in \cref{sec:gptq intro}, falls into this regime.} This {layer-wise} approach leads to substantially more manageable optimization problems and enables efficient training-free compression of very large models \cite{frantar2022gptq, frantar2023sparsegpt, chen2021drone, zhang2024magr}. Throughout this paper, we will use $X \in \mathbb{R}^{m \times N}$ to denote the input calibration dataset of $m$ samples (\textit{e.g.}, tokens) stored in rows, resulting in the loss function $\|XW-X\widetilde{W}\|_F^2$.

Thus, in the case of $\widetilde{W}=Q+LR$, the optimization problem becomes 
\begin{equation}\label{eq:layer_err}
    \underset{Q \in \mathcal{A}^{N \times N'}, L\in\R^{N\times r}, R\in\R^{r\times N'}}{\mathrm{minimize}} \ \|XW - X(Q+LR)\|_F^2.
\end{equation}
Here we slightly abuse our notation by using the same symbol $\mathcal{A}$ for all columns (or blocks of entries within columns), when each column (or block) may have its own grid (i.e. zero point and scaling constant).

\subsection{Contributions and Roadmap} \label{sec:contributions}

We began by formulating the layer-wise low-precision plus low-rank compression problem as the reconstruction objective
    \eqref{eq:layer_err}, where a weight matrix \(W\) is approximated by \(Q+LR\). With that formulation in hand, our main contributions are as follows.

\begin{itemize}[leftmargin=*]

    \item {\bf  Information-theoretic lower bounds.} We prove information-theoretic lower bounds for approximation by low-precision plus low-rank representations. We first show in \cref{thm:density} that an unconstrained integer alphabet together with a rank-one correction is already dense, which explains why finite-alphabet and bounded-magnitude assumptions are necessary for meaningful lower bounds. Under these constraints, \cref{thm:fancy lower bound} gives a worst-case lower bound for general matrices, \cref{thm:fancy lb corollary} extends the bound to a non-spiky class of matrices, and \cref{thm:data lower bound} converts this into a lower bound for the layer-wise reconstruction error.
    \item {\bf GPTQ-intrinsic LoRA.} After reviewing GPTQ \cite{frantar2022gptq} in \cref{sec:gptq intro}, we introduce GPTQ-intrinsic LoRA, a training-free algorithm that incorporates the low-rank correction directly into a GPTQ-style quantization pass. The method chooses the left factor \(L=V_r\), where \(V_r\) contains the top right singular vectors of the calibration matrix \(X\), and then runs GPTQ on an augmented Hessian to produce both the low-precision component \(Q\) and the error absorption factor \(R\). The construction is described in \cref{sec:method,lora_in_GPTQ}.

        \item{\bf Upper bounds on the error.} We prove layer-wise reconstruction error bounds for GPTQ-intrinsic LoRA. The channel-wise bound in \cref{thm:main_channel}, its unregularized least squares analogue in \cref{thm:no_lam_cha}, and the layer-wise bounds in \cref{cor:no_lam_layer,thm:main_layer} show that the usual GPTQ dependence on \(\|X\|_F^2\) can be replaced by the rank-\(r\) residual \(\|X-X_r\|_F^2\), up to the regularization contribution. We also discuss in \cref{finite bit remark} how these infinite-alphabet estimates should be interpreted under a finite quantization alphabet.

    \item{\bf Near optimality of GPTQ-intrinsic LoRA.} We compare the lower and upper bounds in \cref{sec:bound is tight}. Under an approximate rank-\(r\) structure for the calibration matrix and a non-spikiness assumption on \(W\), the upper bound for GPTQ-intrinsic LoRA matches the information-theoretic lower bound in its dominant dependence on the alphabet size and on the smallest singular value of \(X\), up to constants and mild factors.

    \item{\bf Bid-Up.} We propose Bid-Up, a fixed-grid coordinate refinement step for the quantized component \(Q\), and combine it with the {O}ptimal {L}ow-{r}ank {C}ompensation (OLrC) technique from \cref{OLrC} to refine both the low-precision and low-rank components. Each OLrC step optimizes the low-rank compensation for fixed \(Q\), while each Bid-Up coordinate update exactly minimizes over the fixed quantization grid, so a full refinement loop cannot increase the layer-wise reconstruction error, but may decrease it. The method is described in \cref{sec:bid-up,Bid-up}.

    \item{\bf Numerical Validation.} We validate the method empirically on language models and vision transformers. The Qwen3 experiments in \cref{tbl:qwen3} show that GPTQ-intrinsic LoRA improves over GPTQ and over GPTQ followed by OLrC, especially in low-bit regimes. The DeiT-B and DeiT-III-L experiments in \cref{tab1,tab2} show analogous gains for vision transformers, as well as further improvements from additional OLrC and Bid-Up refinement loops.

\end{itemize}

\section{Related Work}\label{sec:background}
To the best of our knowledge, there has been very limited theoretical investigation of the optimization problem \eqref{eq:layer_err}, and of error bounds associated with methods that approximate its solution. In contrast, a substantial body of empirical and algorithmic work has been developed. Important works applying LoRA  style finetuning to quantization include QLoRA \citep{dettmers2024qlora} and QA-LoRA\cite{xu2023qa}. They follow the original LoRA's\cite{hu2021lora}  “Noise \& Zeros” initialization scheme with a focus on reducing memory usage during finetuning and improving inference efficiency. LoftQ \citep{li2023loftq} and LQ-LoRA \cite{guo2023lq} initialize the low-rank factors $L$ and $R$ so that $LR$ is the best low-rank approximation of $W\!-\!Q$ given by singular value decomposition (SVD, by e.g. Eckart-Young theorem\cite{eckart1936approximation}) and finetune the low-rank factors from that. The idea of using the SVD to shrink the gap between $Q$ and $W$ was introduced earlier in \citep{yao2023zeroquant} as Low Rank Compensation (LoRC). SVDQuant \citep{li2024svdqunat} explores a similar approach in a different order, where one first computes an SVD to find a low-rank approximation $LR$ for $W$ and then quantizes the remainder $W-LR$ to obtain $Q$. 

{\bf Calibrated LoRA Initialization. }Going beyond initializations that are independent of calibration data, LQER \cite{zhang2024lqer} leverages an activation induced scale matrix to initialize the low-rank factors. To further enhance initialization, more recent works such as QERA \cite{zhang2024qera}, CALDERA \cite{saha2024compressing}, EoRA \cite{liu2024eora} and CLoQ \cite{deng2025cloq} have all converged to an approach involving approximately solving the layer-wise reconstruction problem \eqref{eq:layer_err}. 
A common theme in these four methods is that the quantized matrix $Q$ is fixed first, after which the low-rank factors are obtained by solving the rank-constrained reconstruction problem
\begin{equation}\label{eq:gsvd_given_Q}
    \underset{L\in\R^{N\times r}, R\in\R^{r\times N'}}{\mathrm{minimize}} \ \|X(W-Q) - XLR\|_F^2,
\end{equation}
for a prescribed rank $r$ and calibration matrix $X$.

In particular,  \cite{saha2024compressing, liu2024eora, deng2025cloq} obtain $Q$ via the GPTQ\cite{frantar2022gptq} or LDLQ\cite{chee2023quip} algorithms (which are known to be equivalent). The rank constrained regression problem \eqref{eq:gsvd_given_Q} admits an explicit solution via the generalized singular value decomposition (GSVD) \cite{takane2001constrained}. The use of GSVD for low-rank approximation in neural networks appears at least as early as \cite{zhang2015accelerating}. Assuming that $X^\top X$ is invertible, an optimal product is given by
\begin{equation}\label{eq:gsvd_sol}
    L^\star R^\star
    =
    (X^\top X)^{-\frac{1}{2}}
    \mathcal{T}_r\big((X^\top X)^{\frac{1}{2}}(W-Q)\big),
\end{equation}
where $\mathcal{T}_r$ returns a best rank-$r$ approximation, computed by truncated SVD. We summarize this solution to \eqref{eq:gsvd_given_Q} in \cref{OLrC} and present more background on the GSVD in \cref{sec:GSVD}. Consequently, the minimum value in \eqref{eq:gsvd_given_Q} is characterized by the following proposition.

\begin{proposition}\label{prop:gsvd}
    Let $M\in\R^{N\times N'}$ and $X\in\R^{m\times N}$, and assume that $H=X^\top X$ is invertible. Then
    \[
    \underset{L\in\R^{N\times r}, R\in\R^{r\times N'}}{\min}
    \|XM-XLR\|_F^2
    =
    \|H^{\frac{1}{2}}M-\mathcal{T}_r(H^{\frac{1}{2}}M)\|_F^2.
    \]
\end{proposition}

\begin{algorithm}[t]
\caption{OLrC (\textbf{O}ptimal \textbf{L}ow-\textbf{r}ank \textbf{C}ompensation)}
\begin{algorithmic}\label{OLrC}
\STATE \textbf{Input:} $H=X^\top X, W, Q, r$
\STATE \(U, \Sigma, V = \mathrm{
svd}\left(H^\frac{1}{2}(W-Q)\right)\)\hfill{Compute singular value decomposition}

\STATE\(U\gets H^{-\frac{1}{2}}U\)\hfill{Scale $H^\frac{1}{2}$ back to finish GSVD}

\STATE\(L = U_r\Sigma_r, R = V_r^\top\)\hfill{Truncate to rank $r$}

\STATE \textbf{return }{$L$, $R$}
\end{algorithmic}
\end{algorithm}

This closed-form characterization is a common mathematical core underlying these methods.
Indeed, the equivalence between the low-rank factors derived in QERA (Theorem 1 in \cite{zhang2024qera}) and CALDERA (Lemma 4.2 in \cite{saha2024compressing}) has already been discussed in \citet[Appendix A.3]{zhang2024qera}.\footnote{One caveat is that CALDERA allows the low-rank factors $L$ and $R$ to be generated directly in low precision.} One can similarly verify that CLoQ (Theorem 3.1 in \cite{deng2025cloq}) and EoRA (Theorem 1 in \cite{liu2024eora}) lead to the same mathematical formulation.

ApiQ \cite{liao2024apiq} targets a modified version of \eqref{eq:layer_err} that uses two calibration matrices, one sampled from the full-precision model and one sampled from the quantized model. It uses a back-propagation-based procedure to minimize the layer-wise, or block-wise, reconstruction error by jointly training the low-rank factors $L,R$ and the quantization parameters for $Q$. As a result, ApiQ is more computationally expensive than the closed-form initialization methods discussed above and does not provide comparable theoretical guarantees.
As a final example, SRR \cite{cho2026preserve} instead splits the rank budget $r$ into two parts and computes low-rank factors both before and after quantization.

\section{Information Theoretic Lower Bounds}\label{sec:lower bound}

In this section, we study the smallest possible value of 
\[
    \|XW - X(Q+LR)\|_F^2
\]
over all choices of the low-precision component $Q$ and the low-rank correction $LR$, regardless of the algorithm used to construct them. To obtain such lower bounds, we first study the simpler approximation error
\(
    \|W-(Q+LR)\|_F,
\)
and then derive corresponding bounds for
\(
    \|XW - X(Q+LR)\|_F^2
\)
as corollaries.

We begin with a warm-up result that may initially appear counterintuitive. Any real-valued matrix can be approximated arbitrarily well by the sum of a rank-one matrix and an integer-valued matrix, as the theorem below shows.

\begin{theorem}\label{thm:density}
    Let $W \in \mathbb{R}^{N \times N'}$ be any matrix, and let $\delta>0$. Then, for every $\epsilon>0$, there exist $Z \in \mathbb{Z}^{N \times N'}$, $a\in\mathbb{R}^N$, and $b\in\mathbb{R}^{N'}$ such that
    \[
        \| W - (\delta Z + ab^\top)\|_{\mathrm{op}} < \epsilon.
    \]
    Equivalently, the set
    \[
        \mathbb{S}_\delta^{1}
        :=
        \{\delta Z+\Lambda :
        Z\in \mathbb{Z}^{N \times N'},\,
        \Lambda \in \mathbb{R}^{N \times N'},\,
        \mathrm{rank}(\Lambda)\RSnew{\leq} 1\}
    \]
    is dense in $\mathbb{R}^{N \times N'}$.
\end{theorem}

The proof is deferred to \cref{sec:proofs}. The following corollary is immediate.

\begin{corollary}
    For any $X\in\R^{m\times N}$ and $\delta>0$, the set
    \[
        \{X(\delta Q+\Lambda) :
        Q\in \mathbb{Z}^{N \times N'},\,
        \Lambda \in \mathbb{R}^{N \times N'},\,
        \mathrm{rank}(\Lambda)\RSnew{\leq} 1\}
    \]
    is dense in
    \(
        \{XW : W \in \mathbb{R}^{N \times N'}\}.
    \)
\end{corollary}

Since the set 
\(
    \mathbb{S}_\delta^{1},
\)
defined in \cref{thm:density}, is dense in $\mathbb{R}^{N \times N'}$, it follows immediately that for any $r\ge 1$ the larger set
\(
    \mathbb{S}_\delta^{\le r}
    =
    \{\delta Q+\Lambda :
    Q\in \mathbb{Z}^{N \times N'},\,
    \Lambda \in \mathbb{R}^{N \times N'},\,
    \mathrm{rank}(\Lambda)\le r\}
\)
is also dense in $\mathbb{R}^{N \times N'}$.

{This observation is mathematically striking and it suggests that even a rank-one correction can absorb all the quantization error, but it does not reflect what one observes in practice.} In practice, a rank-one correction is not sufficient to recover the performance of a heavily quantized model, and increasing the rank does lead to meaningful empirical improvements. The reason is that the density result exploits the unboundedness of $\mathbb{Z}^{N \times N'}$. In practical compression, the quantized component $Q$ must be represented using a finite number of bits, and the entries of the low-rank correction $\Lambda$ cannot be arbitrarily large. Thus, a more realistic 
model should constrain the number of values available to each entry of $Q$, while still allowing the grid spacing $\delta$ to adapt to the scale of $W$. It should also impose a magnitude constraint on the low-rank component. This motivates the following definition.

\begin{definition}
    For positive integers $r$ and $B$, and a positive real number $\rho$, define the constrained low-precision plus low-rank class by
    \[
        \mathbb{S}_{B}^{r,\rho}
        :=
        \left\{
        \delta Q+\Lambda :
        \delta>0,\ 
        Q\in \mathbb{Z}^{N \times N'},\ 
        \|Q\|_{\infty}\leq B,\ 
        \Lambda \in \mathbb{R}^{N \times N'},\ 
        \operatorname{rank}(\Lambda)\leq r,\ 
        \|\Lambda\|_F\leq \rho
        \right\}.
    \]
    Here $\|Q\|_\infty$ denotes the entrywise infinity norm, so $Q$ uses a scaled alphabet of size $2B+1$, while $\rho$ controls the size of the low-rank correction.
\end{definition}

With this definition, the quantized component $Q$ is restricted to a finite alphabet, while the grid spacing $\delta>0$ remains free to adapt to the scale of the matrix being approximated. The low-rank component $\Lambda$ is also required to have bounded Frobenius norm.
To normalize the problem, we restrict our attention to matrices in the unit Frobenius ball \(\mathbb{B}_F^1 := \{A\in\mathbb{R}^{N\times N'} : \|A\|_F\leq 1\}\). The question is whether the constrained class \(\mathbb{S}_{B}^{r,\rho}\) can still be dense in this unit ball. Suppose \(A\in\mathbb{B}_F^1\), \(\epsilon\in(0,1)\), and \(M\in\mathbb{S}_{B}^{r,\rho}\) satisfies \(\|M-A\|_F\leq \epsilon\). Then, by the triangle inequality, \(\|M\|_F\leq \|A\|_F+\epsilon<2\). Thus, when studying \(\epsilon\)-approximations to matrices in \(\mathbb{B}_F^1\), it suffices to consider elements in \(\mathbb{S}_{B}^{r,\rho}\cap \mathbb{B}_F^2\).

We need the following auxiliary lemma on the covering number of low-rank matrices with bounded Frobenius norm. Its proof, which is standard,  can be found in \cref{sec:proofs}.

\begin{lemma}\label{lemma:cover lr mat}
    Let
    \(
        \mathbb{S}^{r,\rho}
        :=
        \{\Lambda\in\R^{N\times N'} :
        \operatorname{rank}(\Lambda)\leq r,\ 
        \|\Lambda\|_F\leq \rho\}.
    \)
    Then, for $0<\epsilon\leq \rho$, its covering number in the Frobenius metric satisfies
    \[
        \mathcal{N}(\epsilon,\|\cdot\|_F,\mathbb{S}^{r,\rho})
        \leq
        \left(\frac{10\rho}{\epsilon}\right)^{r(N+N'+1)+1}.
    \]
\end{lemma}

Now we state and prove the main result of this section.
\begin{theorem}[Information-theoretic Lower Bound on Approximating General Matrices]\label{thm:fancy lower bound}~\\
Assume that $r(N+N'+1)+2 < NN'$, $NN'>28$ and \SZnew{$\rho>\frac{1}{9}$}. Then
\[
\underset{A\in\mathbb{B}_F^1}{\sup}\quad\underset{M\in\mathbb{S}_{B}^{r,\rho}\cap \mathbb{B}_F^2}{\inf}\|M-A\|_F \geq \frac{20}{41}
    \cdot \dfrac{1}{(2B+1)^{\frac{J}{J-1}}(41\rho)^{\frac{1}{J-1}}},
\]
where \[J=J(r,N,N')=\frac{NN'}{r(N+N'+1)+2}\in(1,+\infty).\]
\end{theorem}

Before proving the theorem, we provide a brief interpretation of the lower bound.

\begin{remark}
    The parameter $J$ measures how small the allowed low rank component is relative to the ambient dimension. Indeed, larger $J$ corresponds to a smaller rank budget relative to the size of the matrix. Under the assumptions of the theorem, the right side of the bound increases with $J$. Thus, as the allowed low rank component becomes more restrictive, the worst case approximation error necessarily becomes larger.

    Recall also that $2B+1$ is the size of the integer alphabet $\{-B,\ldots,-1,0,1,\ldots,B\}$ used for each entry of $Q$, up to the scaling factor $\delta$. The theorem shows that the lower bound scales as $\rho^{-1/(J-1)}$ in the magnitude budget for the low rank component, and as $(2B+1)^{-J/(J-1)}$ in the alphabet size. Thus, increasing either the allowed low rank magnitude or the quantization alphabet improves the best possible approximation, but with rates controlled by the effective compression parameter $J$.
\end{remark}

Now we prove \cref{thm:fancy lower bound}.

\begin{proof}
 It suffices to show that, for every $\epsilon$ below the claimed lower bound, there exists $A^\star\in \mathbb{B}_F^1$ such that
    \(
        \operatorname{dist}\bigl(A^\star,\mathbb{S}_{B}^{r,\rho}\cap \mathbb{B}_F^2\bigr)>\epsilon.
    \)
    For $\epsilon\in(0,1)$, define
    \[
        U_\epsilon
        :=
        \left\{
        A\in\mathbb{B}_F^1 :
        \operatorname{dist}\bigl(A,\mathbb{S}_{B}^{r,\rho}\cap \mathbb{B}_F^2\bigr)\leq \epsilon
        \right\}.
    \]
    Thus $U_\epsilon$ consists of matrices in the unit Frobenius ball that can be approximated to accuracy $\epsilon$ by an element of $\mathbb{S}_{B}^{r,\rho}\cap \mathbb{B}_F^2$. If we can show that $\operatorname{Vol}(U_\epsilon)<\operatorname{Vol}(\mathbb{B}_F^1)$, then some matrix in $\mathbb{B}_F^1$ lies outside $U_\epsilon$, giving the desired lower bound.

    Let $\mathcal{N}(\epsilon,\|\cdot\|_F,\mathbb{S}_{B}^{r,\rho}\cap \mathbb{B}_F^2)$ be the covering number of $\mathbb{S}_{B}^{r,\rho}\cap \mathbb{B}_F^2$ in Frobenius norm. If this set is covered by Frobenius balls of radius $\epsilon$, then enlarging those balls to radius $2\epsilon$ covers $U_\epsilon$. Therefore
    \[
        \operatorname{Vol}(U_\epsilon)
        \leq
        \mathcal{N}(\epsilon,\|\cdot\|_F,\mathbb{S}_{B}^{r,\rho}\cap \mathbb{B}_F^2)
        (2\epsilon)^{NN'}
        \operatorname{Vol}(\mathbb{B}_F^1).
    \]
    Thus it is enough to guarantee
    \[
        \mathcal{N}(\epsilon,\|\cdot\|_F,\mathbb{S}_{B}^{r,\rho}\cap \mathbb{B}_F^2)
        (2\epsilon)^{NN'}
        <\frac{1}{2}.
    \]

 It remains to bound the covering number. For each fixed $Q\in\mathbb{Z}^{N\times N'}$ with $\|Q\|_\infty\leq B$, define
    \[
        \mathbb{S}(Q)
        :=
        \left\{
        \delta Q+\Lambda :
        \delta>0,\ 
        \operatorname{rank}(\Lambda)\leq r,\ 
        \|\Lambda\|_F\leq \rho,\ 
        \|\delta Q+\Lambda\|_F\leq 2
        \right\}.
    \]
    Then
    \(
        \mathbb{S}_{B}^{r,\rho}\cap \mathbb{B}_F^2
        =
        \bigcup_{\substack{Q\in\mathbb{Z}^{N\times N'}\\ \|Q\|_\infty\leq B}}
        \mathbb{S}(Q).
    \)
    There are at most $(2B+1)^{NN'}$ such matrices $Q$. Hence
    \[
        \mathcal{N}(\epsilon,\|\cdot\|_F,\mathbb{S}_{B}^{r,\rho}\cap \mathbb{B}_F^2)
        \leq
        (2B+1)^{NN'} T,
    \]
    where $T$ is an upper bound for $\mathcal{N}(\epsilon,\|\cdot\|_F,\mathbb{S}(Q))$.
  For fixed $Q$, every element of $\mathbb{S}(Q)$ can be written as $\delta Q+\Lambda$, with $\Lambda\in\mathbb{S}^{r,\rho}$ and    \[        \|\delta Q\|_F        \leq        \|\delta Q+\Lambda\|_F+\|\Lambda\|_F        \leq        2+\rho.    \]    Thus    \[        \mathbb{S}(Q)        \subseteq        \mathbb{S}^{r,\rho}        +        \{\delta Q:\delta>0,\ \|\delta Q\|_F\leq 2+\rho\}.    \]    The second set is a line segment of length at most $2+\rho$, and \cref{lemma:cover lr mat} controls the first set. Using     \[        \mathcal{N}(\epsilon,\|\cdot\|,A+B)        \leq        \mathcal{N}\left(\frac{\epsilon}{2},\|\cdot\|,A\right)        \mathcal{N}\left(\frac{\epsilon}{2},\|\cdot\|,B\right),    \]    we obtain    \[        \mathcal{N}(\epsilon,\|\cdot\|_F,\mathbb{S}(Q))        \leq        \left\lceil\frac{2+\rho}{\epsilon}\right\rceil        \left(\frac{20\rho}{\epsilon}\right)^{r(N+N'+1)+1}.    \]    
  
  \SZnew{Since $\rho>1/9$ and $\epsilon\leq\rho$, we have $2+\rho+\epsilon\leq2+2\rho\leq 20\rho$}, and therefore    \[        \mathcal{N}(\epsilon,\|\cdot\|_F,\mathbb{S}_{B}^{r,\rho}\cap \mathbb{B}_F^2)        \leq        (2B+1)^{NN'}        \left(\frac{20\rho}{\epsilon}\right)^{r(N+N'+1)+2}.    \]    Substituting this into the preceding sufficient condition, it is enough to require    \[        (2B+1)^{NN'}        \left(\frac{20\rho}{\epsilon}\right)^{r(N+N'+1)+2}        (2\epsilon)^{NN'}        <\frac12.    \]    
  Rearranging and taking logarithms gives the sufficient condition    \[        \log_2\frac{\epsilon}{20\rho}        <        -\frac{1}{NN'-[r(N+N'+1)+2]}        -\frac{NN'\log_2((2B+1)40\rho)}{NN'-[r(N+N'+1)+2]}.    \]    Since    \(       J=\frac{NN'}{r(N+N'+1)+2},    \)    this is implied by    \(       \epsilon        <        20\rho        \left((2B+1)40\rho\cdot 2^{1/(NN')}\right)^{-J/(J-1)}.    \)    Finally, $NN'>28$ implies $40<40\cdot 2^{1/(NN')}<41$. Since $J/(J-1)>1$, it suffices to take    \[        \epsilon        <        20\rho        \left((2B+1)41\rho\right)^{-J/(J-1)}        =        \frac{20}{41}        \cdot        \frac{1}{(2B+1)^{J/(J-1)}(41\rho)^{1/(J-1)}}.    \]    For every $\epsilon$ below this value, $\operatorname{Vol}(U_\epsilon)<\frac12\operatorname{Vol}(\mathbb{B}_F^1)$. Hence some $A^\star\in\mathbb{B}_F^1$ satisfies    \(        \operatorname{dist}\bigl(A^\star,\mathbb{S}_{B}^{r,\rho}\cap\mathbb{B}_F^2\bigr)>\epsilon.    \)    Taking $\epsilon$ arbitrarily close to the displayed threshold proves the claim.
\end{proof}

\begin{remark}\label{why flat}
    The same lower bound remains valid even if the supremum over $\mathbb{B}_F^1$ is restricted to a class of non-spiky matrices. More precisely, for an absolute constant $C$, the same bound holds with $\mathbb{B}_F^1$ replaced by
    \[
        \mathbb{B}_F^1\cap
        \left\{
        A\in\mathbb{R}^{N\times N'} :
        \|A\|_\infty\leq C\frac{\log(NN')}{\sqrt{NN'}}
        \right\}.
    \]

Indeed, let $d=NN'$. After vectorizing matrices in $\mathbb{R}^{N\times N'}$ as vectors in $\mathbb{R}^d$, and rescaling by $\sqrt d$, we obtain
    \[
        \frac{
        \operatorname{Vol}\left(
        \mathbb{B}_F^1\cap
        \left\{
        A\in\mathbb{R}^{N\times N'} :
        \|A\|_\infty\leq C\frac{\log(NN')}{\sqrt{NN'}}
        \right\}
        \right)}
        {\operatorname{Vol}(\mathbb{B}_F^1)}
        =
        \mathbb{P}_{x\sim \operatorname{Unif}(\mathbb{B}(0,\sqrt d))}
        \left(
        \|x\|_\infty\leq C\log d
        \right).
    \]

 The uniform distribution on the Euclidean ball $\mathbb{B}(0,\sqrt d)\subset\mathbb{R}^d$ is subgaussian \cite{vershynin2020high}. Therefore, by choosing $C$ sufficiently large, the probability on the right is at least $1-1/d$, and hence at least $3/4$ when $N,N'\geq 2$. Since the proof above shows that $\operatorname{Vol}(U_\epsilon)<\frac12\operatorname{Vol}(\mathbb{B}_F^1)$, it follows that
    \[
        \operatorname{Vol}(U_\epsilon)
        <
        \frac{2}{3}
        \operatorname{Vol}\left(
        \mathbb{B}_F^1\cap
        \left\{
        A\in\mathbb{R}^{N\times N'} :
        \|A\|_\infty\leq C\frac{\log(NN')}{\sqrt{NN'}}
        \right\}
        \right).
    \]
    Thus this restricted class still contains a matrix at distance larger than $\epsilon$ from $\mathbb{S}_{B}^{r,\rho}\cap\mathbb{B}_F^2$, and the same lower bound holds. In other words, $\mathbb{S}_{B}^{r,\rho}$ is not dense even in a non-spiky subset of the unit Frobenius ball. Non-spiky matrices are especially relevant in post training quantization. A recent line of work applies preprocessing transformations that exploit invariances in the computation graph, often of the form $XW=(XT^{-1})(TW)$. The most common examples are rotations, where $T^{-1}=T^\top$, including Hadamard rotations \cite{ashkboos2024quarot}, random rotations \cite{chee2023quip}, and learned rotations \cite{liu2024spinquant}. These transformations are designed to reduce the dynamic range of $W$, thereby promoting flatness, or incoherence, of the transformed weights. Post training compression is then applied to the transformed matrix $TW$.

    With a  more refined volume estimate, we can strengthen this observation to an even flatter class of matrices in the following corollary. The proof is deferred to \cref{sec:proofs}.

\end{remark}

\begin{corollary}[Information Theoretic Lower Bound for Non-spiky Matrices]\label{thm:fancy lb corollary}
    Let
    \[
        \mathbb{B}_{\mathrm{flat}}^1
        :=
        \mathbb{B}_F^1
        \cap
        \left\{
        A\in\mathbb{R}^{N\times N'} :
        \|A\|_\infty\leq \frac{\sqrt{3}}{\sqrt{NN'}}
        \right\}.
    \]
    Assume that $r(N+N'+1)+2 < NN'$, $NN'\geq 144$, and \SZnew{$\rho>\frac{1}{9}$}. Then
    \[
        \sup_{A\in\mathbb{B}_{\mathrm{flat}}^1}
        \inf_{M\in\mathbb{S}_{B}^{r,\rho}\cap \mathbb{B}_F^2}
        \|M-A\|_F
        \geq
        \frac{20}{41}
        \cdot
        \frac{1}{(2B+1)^{\frac{J}{J-1}}
        \left(41\sqrt{\frac{\pi e}{6}}\rho\right)^{\frac{1}{J-1}}},
    \]
    where
    \(
        J=J(r,N,N')=\frac{NN'}{r(N+N'+1)+2}\in(1,\infty).
    \)
\end{corollary}

As a direct consequence, we obtain the corresponding lower bound for layer-wise reconstruction error.

\begin{corollary}\label{thm:data lower bound}
    Let $X\in\R^{m\times N}$, with $m>N$, and assume that $\sigma_{\min}(X)>0$. Then there exists $A\in\R^{N\times N'}$ satisfying $\|A\|_F\leq1$ and $\|A\|_\infty\leq \sqrt{3}/\sqrt{NN'}$ such that
    \[
        \inf_{M\in\mathbb{S}_{B}^{r,\rho}\cap \mathbb{B}_F^2}
        \|X(M-A)\|_F
        \geq
        \frac{20}{41}
        \cdot
        \sigma_{\min}(X)
        \cdot
        \frac{1}{(2B+1)^{\frac{J}{J-1}}
        \left(41\sqrt{\frac{\pi e}{6}}\rho\right)^{\frac{1}{J-1}}}.
    \]
\end{corollary}

\section{GPTQ-intrinsic LoRA}
Having presented lower bounds associated with solving \eqref{eq:layer_err}, we now propose and analyze a novel algorithm for solving it, {GPTQ-intrinsic LoRA}. To that end, we first provide a brief overview of GPTQ in \cref{sec:gptq intro}, recalling its development from early pruning methods to its current modern formulation. After that, we  formally introduce our method in \cref{sec:method} and provide upper bounds on its associated layer-wise reconstruction error assuming an infinite alphabet. We also provide, in \cref{sec:upper bound}, a remark on interpreting the bounds under a finite alphabet. Finally, we compare the upper bounds in \cref{sec:upper bound} with the information-theoretic lower bounds established in \cref{sec:lower bound} and demonstrate that GPTQ-intrinsic LoRA is near-optimal.

\subsection{An  Introduction to GPTQ}\label{sec:gptq intro}
\paragraph{From pruning to GPTQ.} As the name suggests, our GPTQ-intrinsic LoRA builds upon GPTQ, a widely used baseline in many recent works on PTQ. GPTQ and related algorithms \cite{frantar2022optimal, frantar2023sparsegpt, frantar2022gptq} are developed from an older framework that traces back to the Optimal Brain Damage (OBD, \citet{lecun1989optimal}) and Optimal Brain Surgery (OBS, \citet{hassibi1993optimal}) algorithms, originally designed for network pruning. In OBS, pruning is performed iteratively by solving a small optimization problem at each step. Denoting the “Hessian” by $H = X^\top X$ and letting $\delta_w$ represent the update to the weight vector, the OBS pruning step solves
\[
\underset{\delta_w}{\min} \ \frac{1}{2} \delta_w^\top H \delta_w \quad \text{subject to} \quad e_p^\top \delta_w + w_p = 0, \quad \delta_w|_F = 0,
\]
where $e_p$ is the standard basis vector selecting the $p$-th coordinate to prune, and $\delta_w|_F = 0$ enforces no change to already-fixed coordinates indexed by $F$ (see \cite{hassibi1992second}). This paradigm was adapted to PTQ in Optimal Brain Quantization (OBQ, \citet{frantar2022optimal}). Similar to OBS,  each step of OBQ involves solving
\[
\underset{q \in \mathcal{A}}{\min} \left\{ \ \underset{\delta_w}{\min} \ \frac{1}{2} \delta_w^\top H \delta_w \ \text{subject to} \ e_p^\top \delta_w + w_p = q, \quad \delta_w|_F = 0 \ \right\},
\]
where $\mathcal{A}$ is the quantization alphabet.

In the pruning case, the constrained quadratic problem admits a closed-form solution via the stationary point of its Lagrangian. In the quantization setting, the inner minimization problem remains convex and can be solved in exactly the same way. What is new in quantization is that one must evaluate the objective over all possible choices of $q$ in a discrete set $\mathcal{A}$ and select the minimizer to solve the outer minimization problem. This leads to a local greedy algorithm that alternates between quantizing a coordinate of $w$ and updating the remaining unquantized coordinates to compensate for the induced error, until all the coordinates are quantized. This greedy strategy is applied to all the weight vectors (\textit{i.e.} columns of $W$) in a layer in parallel and repeated layer-wise sequentially. 

Following OBQ, GPTQ (\cref{GPTQ}, \citet{frantar2022gptq}) leverages state-of-the-art Cholesky kernels to improve efficiency and stability, enabling it to scale to very large LLMs. Specifically, GPTQ uses the lower-triangular Cholesky factor $\Psi$ in the decomposition $H^{-1}=\Psi\Psi^T$ in place of $H^{-1}$ itself as it gives a computationally equivalent output when the Cholesky decomposition exists. A ``dampening" term $\lambda I$ is added to $X^TX$ when computing the inverse Hessian to mitigate numerical instability\footnote{Another implementation detail, called lazy batch updates in \cite{frantar2022gptq}, involves processing the weights in blocks of size $B$ to enhance the compute-to-memory-access ratio while preserving the algorithm's mathematical equivalence to the $B=1$ case. We always consider $B=1$ in this paper.}.

\paragraph{GPTQ notation and error diffusion.}
We now fix notation for the GPTQ iterations applied to a single weight vector $w=(w_1,\ldots,w_N)^\top\in\mathbb{R}^N$ using a calibration matrix $X$. Since GPTQ processes the columns of $W$ independently in parallel, it suffices to describe the iteration for one column $w$. Recall that the rows of $X\in\mathbb{R}^{m\times N}$ correspond to calibration samples, such as tokens. We write the feature-wise decomposition of $X$ as $X=\begin{pmatrix}X_1 & \cdots & X_N\end{pmatrix}$, where $X_j\in\mathbb{R}^m$ denotes the $j$-th feature column.

\begin{algorithm}[t]
\caption{GPTQ: Quantize a layer \( W \) to \( Q \) using $H=X^\top X$ with regularization constant $\lambda$}
\begin{algorithmic}[1]\label{GPTQ}
\STATE  \( (H+\lambda I)^{-1}=\Psi\Psi^\top \) \hfill {Perform Cholesky decomposition}
\FOR{every column $w$ in $W$ (in parallel)}
\STATE $w^{(0)} {=w}$
\FOR{$t = 1$ to $N$} 
        \STATE
         $q_{t}=\mathcal{Q}(w^{(t-1)}_t)$ 
        \STATE \( w^{(t)}_{\geq t+1}= w^{(t-1)}_{\geq t+1}+(q_{t}-w^{(t-1)}_t)\frac{\Psi_{\geq t+1,t}}{\Psi_{tt}}  \)  
    \ENDFOR
\ENDFOR
\STATE \textbf{return} every \( q \) in \( Q \) 
\end{algorithmic}
\end{algorithm}

At each step $t$, GPTQ rounds the current weight entry using the scalar quantizer $\mathcal{Q}:\mathbb{R}\to\mathcal{A}$, setting $q_t=\mathcal{Q}(w_t^{(t-1)})$. It then updates the remaining unquantized entries  as in line 6 of \cref{GPTQ}. It was shown in \citet{zhang2025qronos} that this update is equivalent to
\[
w^{(t)}_{\geq t+1}
=
\underset{(v_{t+1}, \dots, v_N) \in \mathbb{R}^{N - t}}{\arg\min}
\, \frac{1}{2}
\left\|
(q_t - w^{(t-1)}_t) X_t
+
\sum_{j = t+1}^{N} (v_j - w^{(t-1)}_j) X_j
\right\|^2.
\]
Here $(q_t - w^{(t-1)}_t)X_t$ is the error introduced by rounding the $t$-th entry. The update from $w^{(t-1)}_{\geq t+1}$ to $w^{(t)}_{\geq t+1}$ adjusts the remaining weights to minimize the resulting distortion in $\ell_2$, a process referred to as error diffusion in \cite{zhang2025qronos}. We denote the full state of the algorithm after step $t$ by
\[
w^{(t)} = (q_{\leq t}, w^{(t)}_{\geq t+1})\in \mathcal{A}^{t} \times \mathbb{R}^{N-t},
\]
with initialization $w^{(0)}=w\in\mathbb{R}^N$ and final output $w^{(N)}=q\in\mathcal{A}^N$. We also define the reconstruction error after step $t$ by
\begin{equation}\label{eq:err_def}
{\bf e}_t
=
Xw - Xw^{(t)}
=
Xw
-
\sum_{j=1}^{t} q_j X_j
-
\sum_{j=t+1}^{N} w^{(t)}_j X_j.
\end{equation}

Finally, applying GPTQ with regularization parameter $\lambda>0$, meaning that $\lambda I$ is added to the Hessian $H=X^\top X$, is equivalent to applying unregularized GPTQ to the appended calibration matrix
\[
X_{\mathrm{tall}}
=
\begin{pmatrix}
X \\
\sqrt{\lambda} I
\end{pmatrix}.
\]
Indeed, $X^\top X+\lambda I=X_{\mathrm{tall}}^\top X_{\mathrm{tall}}$. In particular, $X_{\mathrm{tall}}$ has full column rank, regardless of whether $X$ itself is full rank or whether $m\geq N$.

\subsection{The GPTQ-intrinsic LoRA Construction}\label{sec:method}

Unlike existing methods that treat low-rank adaptation as a separate step after PTQ \cite{zhang2024qera, saha2024compressing, deng2025cloq}, our approach incorporates the low-rank correction directly into a GPTQ-style quantization pass. More precisely, for a chosen left factor $L\in\R^{N\times r}$, we construct $Q\in\A^{N\times N'}$ and $R\in\R^{r\times N'}$ jointly, using a completely training-free procedure aimed at reducing the layer-wise reconstruction error $\|XW-X(Q+LR)\|_F^2$.

For any fixed $L$, define $\widehat X=XL$. Then the objective can be rewritten as
\[
\|XW-X(Q+LR)\|_F^2
=
\|XW-(XQ+\widehat X R)\|_F^2
=
\left\|
\left[\begin{array}{c:c}
X & \widehat X
\end{array}\right]
\left[\begin{array}{c}
W \\ \hdashline
0
\end{array}\right]
-
\left[\begin{array}{c:c}
X & \widehat X
\end{array}\right]
\left[\begin{array}{c}
Q \\ \hdashline
R
\end{array}\right]
\right\|_F^2.
\]
Thus, once $L$ is fixed, the low-rank correction can be viewed as adding $r$ extra full-precision columns to the calibration matrix and $r$ corresponding full-precision rows to the quantized weight representation. This reformulation is visualized in \cref{fig:equivalence}.

\begin{figure}[t]
    \centering
    \vspace{-1em}\includegraphics[width=.86\linewidth]{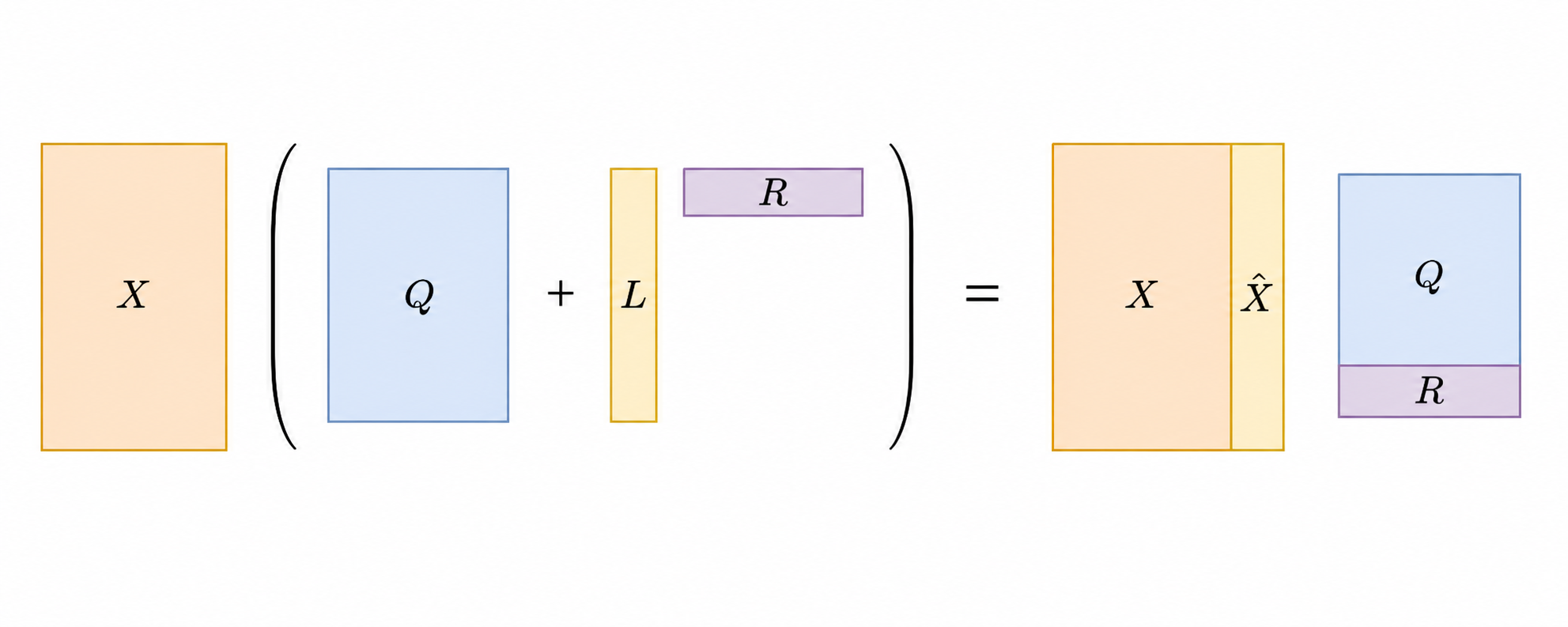}
    \vspace{-3em}
    \caption{Augmenting the data and weight matrices}
    \label{fig:equivalence}
\end{figure}

\paragraph{Generating $R$ given $L$.}
Define the augmented calibration matrix and  the augmented weight matrix 
\[
\mathbb{X}
=
\left[\begin{array}{c:c}
X & \widehat X
\end{array}\right]
\in\R^{m\times (N+r)} \quad \text{ and }
\quad 
\mathbb{W}
=
\left[\begin{array}{c}
W \\ \hdashline
0
\end{array}\right]
\in\R^{(N+r)\times N'}.
\]
Our method entails applying GPTQ to $\mathbb{W}$ using the augmented Hessian $\mathbb{H}=\mathbb{X}^\top\mathbb{X}$, together with the usual dampening term. However, we run GPTQ for only $N$ iterations, so that only the first $N$ rows of $\mathbb{W}$ are quantized. The result is a quantized matrix $Q\in\A^{N\times N'}$, together with a full-precision matrix $R\in\R^{r\times N'}$. Since $R$ is initialized at zero and updated throughout the $N$ GPTQ steps, it absorbs the errors generated while sequentially quantizing $W$ into $Q$. For this reason, we refer to $R$ as the error absorption factor.

\paragraph{Selecting $L$.}It remains to choose the left factor $L$. Our choice is motivated by the fact that, in the optimal low-rank compensation procedure of \cref{OLrC}, the column space of the left low-rank factor is spanned by eigenvectors of $X^\top X$, equivalently by right singular vectors of $X$. In addition, $L$ acts as a dimension reduction map on the features of the calibration matrix $X$. It is therefore natural to choose $L$ using principal component analysis. Specifically, if $X=U\Sigma V^\top$ is a singular value decomposition of $X$, we set $L=V_r\in\R^{N\times r}$, where $V_r$ contains the first $r$ columns of $V$. We refer to $L$ as the feature extraction factor.

\begin{figure}
    \centering
    \includegraphics[width=0.8\linewidth]{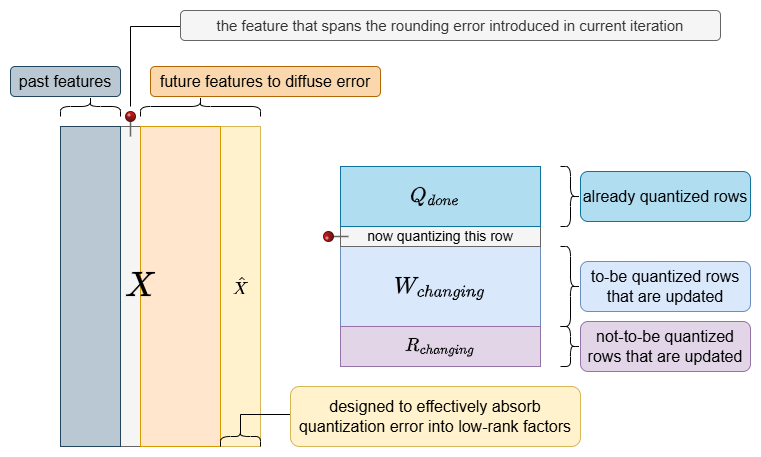}
    \caption{GPTQ-intrinsic LoRA}
    \label{fig:gptqlora}
\end{figure}

With this choice of $L$, we obtain $Q$ and $R$ by running GPTQ for $N$ iterations on the augmented matrix  $\left[\begin{array}{c}
W \\ \hdashline
0
\end{array}\right]$ producing the `mixed-precision' $\left[\begin{array}{c}
Q \\ \hdashline
R
\end{array}\right]$ as described above. A schematic summary of the method is given in \cref{fig:gptqlora}

\paragraph{Hessian-only implementation.}
In many open-source implementations of PTQ algorithms, the calibration data is only used to generate the Hessian $H=X^\top X$ (or its regularized version), for example as a sum of outer products of the samples with themselves. The calibration matrix $X$ itself is not explicitly formed, in order to reduce peak memory usage.
Our construction can be implemented in this setting as well. If $X=U\Sigma V^\top$ is an SVD of $X$, then $H=V\Sigma^2V^\top$, so the choice $L=V_r$ can be obtained directly from the eigen-decomposition of $H$. The augmented Hessian can then be formed as
\[
\mathbb{H}
=
\left[\begin{array}{c:c}
H & HV_r \\ \hdashline
V_r^\top H & V_r^\top H V_r
\end{array}\right],
\]
without explicitly forming either $X$ or $\mathbb{X}$.  We give pseudo code for GPTQ-intrinsic LoRA in \cref{lora_in_GPTQ}.

\begin{remark}[Stable factorization of the augmented Hessian]
Because $\widehat X=XL$, the columns of the augmented calibration matrix $\mathbb{X}=[X\ \widehat X]$ are linearly dependent whenever $r>0$.
Consequently, the augmented Hessian $\mathbb{H}=\mathbb{X}^\top\mathbb{X}$ is singular, even when $H=X^\top X$ is invertible.  Thus, standard GPTQ implementations that compute a Cholesky factor of $(\mathbb{H}+\lambda I)^{-1}$ may become numerically unstable unless one uses a larger regularization parameter $\lambda$. In \cref{lora_in_GPTQ}, we instead use a more robust procedure for computing the required lower-triangular factor $\mathbb{L}$ such that $(\mathbb{H}+\lambda I)^{-1}=\mathbb{L}\mathbb{L^\top}$. Specifically, we use an eigen-decomposition together with a QR decomposition to obtain the factor that would ordinarily be produced by a Cholesky decomposition. This allows us to use smaller values of $\lambda$.  Further implementation details are given in \cref{sec:better cholesky}. We emphasize that as \cref{lora_in_GPTQ} uses only eigen-decompositions and QR decompositions, it is therefore both gradient-free and SVD-free, despite using a LoRA-style parameterization.
\end{remark}

\begin{algorithm}[ht]
\caption{GPTQ-intrinsic LoRA using $H=X^\top X$}
\label{lora_in_GPTQ}
\begin{algorithmic}[1]
\STATE\textbf{Input:} Hessian $H=X^\top X$, weight matrix $W\in\mathbb{R}^{N\times N'}$, rank $r$, regularization parameter $\lambda$, quantizer $\mathcal{Q}$
\STATE Compute $H=VSV^\top$ and set $L=V_r$
\STATE Form the augmented Hessian \qquad\qquad\qquad\quad
\(
    \mathbb{H}
    =
    \begin{bmatrix}
        H & HV_r \\
        V_r^\top H & V_r^\top H V_r
    \end{bmatrix}
\)
\STATE Form the augmented weight matrix
\qquad \qquad \(
    \mathbb{W}
    =
    \begin{bmatrix}
        W\\
        0_{r\times N'}
    \end{bmatrix}
\)
\STATE Compute $\mathbb{L}$ such that $(\mathbb{H}+\lambda I)^{-1}=\mathbb{L}\mathbb{L}^\top$ using the procedure in \cref{sec:better cholesky}

\FOR{each column $\mathbbm{w}$ of $\mathbb{W}$, in parallel}
    \STATE Set $\mathbbm{w}^{(0)}=\mathbbm{w}$ and initialize $\mathbf{q}=0_N$
    \FOR{$t=1,\ldots,N$}
        \STATE $q_t=\mathcal{Q}\bigl(\mathbbm{w}^{(t-1)}_t\bigr)$
        \STATE
        \(
            \mathbbm{w}^{(t)}_{\geq t+1}
            =
            \mathbbm{w}^{(t-1)}_{\geq t+1}
            +
            \bigl(q_t-\mathbbm{w}^{(t-1)}_t\bigr)
            \frac{\mathbb{L}_{\geq t+1,t}}{\mathbb{L}_{tt}}
        \)
    \ENDFOR
    \STATE Set $\mathbf{r}=\mathbbm{w}^{(N)}_{\geq N+1}$
\ENDFOR

\STATE \textbf{return} $Q$, $R$, and $L=V_r$
\end{algorithmic}
\end{algorithm}

\subsection{Upper Bounds on Layer-wise Reconstruction Error}\label{sec:upper bound}

In this section, we prove upper bounds on the layer-wise reconstruction error $\|XW-X(Q+LR)\|_F^2$ for the choice $L=V_r$, together with the matrices $Q$ and $R$ produced by \cref{lora_in_GPTQ}.

The algorithm uses the regularized augmented Hessian $\mathbb{H}+\lambda I_{N+r}=\mathbb{X}^\top\mathbb{X}+\lambda I_{N+r}$. As noted in \cref{sec:gptq intro}, this is equivalent to applying the unregularized algorithm to the appended augmented data matrix
\[
    \mathbb{X}_+
    :=
    \begin{pmatrix}
        \mathbb{X}\\
        \sqrt{\lambda} I_{N+r}
    \end{pmatrix}
    =
    \begin{pmatrix}
        X & \widehat X\\
        \sqrt{\lambda} I_1 & \sqrt{\lambda} I_2
    \end{pmatrix},
\]
since $\mathbb{X}_+^\top\mathbb{X}_+=\mathbb{H}+\lambda I_{N+r}$. Here
\[
    I_1=
    \begin{bmatrix}
        \SZnew{I_{N\times N}}\\
        0_{r\times N}
    \end{bmatrix}
    \in\mathbb{R}^{(N+r)\times N},
    \qquad
    I_2=
    \begin{bmatrix}
        0_{N\times r}\\
        \SZnew{I_{r\times r}}
    \end{bmatrix}
    \in\mathbb{R}^{(N+r)\times r}.
\]
We also denote the two column blocks of $\mathbb{X}_+$ by
\[
    X_+
    :=
    \begin{pmatrix}
        X\\
        \sqrt{\lambda}I_1
    \end{pmatrix}
    \in\mathbb{R}^{(m+N+r)\times N},
    \qquad
    \widehat X_+
    :=
    \begin{pmatrix}
        \widehat X\\
        \sqrt{\lambda}I_2
    \end{pmatrix}
    \in\mathbb{R}^{(m+N+r)\times r}.
\]
Since $X_+^\top X_+=X^\top X+\lambda I_N$, the matrices $X^\top X$ and $X_+^\top X_+$ have the same eigenvectors. Hence $X$ and $X_+$ have the same right singular vectors. In particular, if $X=U\Sigma V^\top$ is an SVD of $X$, then $X_+=U_+\Sigma_+V^\top$, where $\Sigma_+$ has diagonal entries $(\sigma_i^2+\lambda)^{1/2}$.

Because $\lambda>0$, the matrix $\mathbb{X}_+$ has full column rank. We may therefore invoke the following GPTQ error evolution lemma, which is the main ingredient in the reconstruction error bounds below.

\begin{lemma}[\citet{zhang2025provable}]\label{lemma:err evolution}
Let $X \in \mathbb{R}^{m \times N}$ be full rank with $m \geq N$, and let $w \in \mathbb{R}^N$. Running GPTQ with $H = X^\top X$ using an infinite alphabet $\mathcal{A}=\mathcal{A}^{\delta}= \{ \pm k\delta : k \in \mathbb{Z} \}$, the error defined in \eqref{eq:err_def} satisfies
\begin{align}
    {\bf e}_t &= P_{X_{\geq t+1}^{\perp}}(w^{(t-1)}_t - q_t) X_t + {\bf e}_{t-1} \text{\quad and \quad } {\bf e}_N =\sum_{j=1}^{N} P_{X_{\geq j+1}^{\perp}}(w^{(j-1)}_j - q_j) X_j.\label{final error}
\end{align}
Moreover, the resulting quantized vector $q$ satisfies
\begin{gather}
\|X\mathbf{w}-X\mathbf{q}\|_2^2 = \sum_{j=1}^N |w^{(j-1)}_j - q_j|^2 \, \|P_{X_{\geq j+1}^{\perp}} X_j\|_2^2. 
\end{gather}
\end{lemma}

\begin{theorem}\label{thm:main_channel}
    Assume an infinite uniform alphabet $\mathcal{A}=\mathcal{A}^{\delta}= \{ \pm k\delta : k \in \mathbb{Z} \}$ {and $\lambda>0$}. If \cref{lora_in_GPTQ} is applied with $L=V_r$, then, for each column $\mathbf{w}$ of $W$, the corresponding outputs $\mathbf{q}$ and $\mathbf{r}$ satisfy
    \begin{equation}\label{eq:deterministic qlora channel bound}
        \|X\mathbf{w}-X(\mathbf{q}+L\mathbf{r})\|^2
        +\lambda\|\mathbf{w}-\mathbf{q}\|^2
        +\lambda\|\mathbf{r}\|^2
        \leq
        \frac{\delta^2}{4}
        \left(\|X-X_r\|_F^2+(N+r)\lambda\right).
    \end{equation}
\end{theorem}

We defer the full proof to \cref{sec:proofs}. To illustrate the main idea, we first prove the limiting, unregularized case, corresponding to $\lambda=0$, where the argument is cleaner and the role of the low-rank correction is most transparent.

To do this, we use an equivalent least squares formulation of GPTQ. 
Indeed, although the regularized augmented Hessian $\mathbb{H}+\lambda I$ is invertible for $\lambda>0$, the unregularized augmented Hessian $\mathbb{H}$ is generally singular. Therefore, one cannot implement the $\lambda=0$ version of \cref{lora_in_GPTQ} through the usual inverse Hessian or Cholesky based GPTQ updates. 
 The least squares formulation avoids this issue and allows the same error evolution argument to be applied directly to the augmented data matrix $\mathbb{X}$. We recall the relevant least squares formulation below.

\begin{lemma}[\citet{zhang2025qronos}]\label{lemma:lst GPTQ}
    Let $X \in \mathbb{R}^{m \times N}$ have full column rank, with $m \geq N$, and let $w \in \mathbb{R}^N$. Then the GPTQ iterations with $H = X^\top X$ are equivalent to the following least squares updates, with $w^{(0)}_{\geq 1}=w$:
    \begin{align}
        q_t
        &=
        \underset{p \in \mathcal{A}}{\arg\min}
        \frac{1}{2}
        \left\|
        Xw
        -
        \sum_{j=1}^{t-1} q_j X_j
        -
        pX_t
        -
        \sum_{j=t+1}^{N} w^{(t-1)}_j X_j
        \right\|^2,
        \label{eq:q update}
        \\
        w^{(t)}_{\geq t+1}
        &=
        \underset{(v_{t+1}, \dots, v_N) \in \mathbb{R}^{N-t}}{\arg\min}
        \frac{1}{2}
        \left\|
        Xw
        -
        \sum_{j=1}^{t} q_j X_j
        -
        \sum_{j=t+1}^{N} v_j X_j
        \right\|^2.
        \label{eq:w update}
    \end{align}

    Moreover, these least squares updates are well defined even when $X$ is rank deficient or $m<N$, although the minimizer in the second update may not be unique. In that case, one may still implement unregularized GPTQ through \eqref{eq:w update}, and the error evolution identity in \cref{lemma:err evolution} continues to hold for this formulation.
\end{lemma}

We now apply this least squares formulation to the augmented matrix $\mathbb{X}=[X\ \widehat X]$. This gives the $\lambda=0$ analogue of \cref{thm:main_channel}.

\begin{theorem}\label{thm:no_lam_cha}
    Assume an infinite alphabet $\mathcal{A}=\mathcal{A}^{\delta}= \{ \pm k\delta : k \in \mathbb{Z} \}$. Consider the unregularized version of \cref{lora_in_GPTQ} with $\lambda=0$, {implemented through the least squares GPTQ formulation}. Then, for each column $\mathbf{w}$ of $W$, the corresponding outputs $\mathbf{q}$ and $\mathbf{r}$ satisfy
    \[
        \|X\mathbf{w} - X(\mathbf{q}+L\mathbf{r})\|_2^2
        \leq
        \frac{\delta^2}{4} \|X - X_r \|_F^2 .
    \]
\end{theorem}

\begin{proof}
    By \cref{lemma:err evolution}, applied to the augmented matrix $\mathbb{X}=[X\ \widehat X]$, we have
    \[
    \begin{aligned}
        \|X\mathbf{w} - X(\mathbf{q}+L\mathbf{r})\|_2^2
        &=
        \left\|
        \mathbb{X}
        \begin{bmatrix}
            \mathbf{w}\\
            \mathbf{0}
        \end{bmatrix}
        -
        \mathbb{X}
        \begin{bmatrix}
            \mathbf{q}\\
            \mathbf{r}
        \end{bmatrix}
        \right\|_2^2 \\
        &=
        \sum_{j=1}^N |w^{(j-1)}_j - q_j|^2
        \|P_{\mathbb{X}_{\geq j+1}^{\perp}}X_j\|_2^2 \\
        &\leq
        \frac{\delta^2}{4}
        \sum_{j=1}^N
        \|P_{\mathbb{X}_{\geq j+1}^{\perp}}X_j\|_2^2 .
    \end{aligned}
    \]
    For each $j\leq N$, the columns of $\widehat X$ are contained in $\mathbb{X}_{\geq j+1}$. Hence projection onto $\operatorname{col}(\mathbb{X}_{\geq j+1})^\perp$ can only decrease the norm compared with projection onto $\operatorname{col}(\widehat X)^\perp$. Therefore
    \[
        \sum_{j=1}^N
        \|P_{\mathbb{X}_{\geq j+1}^{\perp}}X_j\|_2^2
        \leq
        \sum_{j=1}^N
        \|P_{\widehat X^\perp}X_j\|_2^2 .
    \]
    Since $L=V_r$, we have $\widehat X=XL=XV_r=U_r\Sigma_r$, and therefore $\operatorname{col}(\widehat X)=\operatorname{col}(U_r)$. It follows that
    \[
    \begin{aligned}
         \|X\mathbf{w} - X(\mathbf{q}+L\mathbf{r})\|_2^2 
         &\leq \frac{\delta^2}{4} \sum_{j=1}^N
        \|P_{\widehat X^\perp}X_j\|_2^2
        = \frac{\delta^2}{4}
        \sum_{j=1}^N
        \|P_{U_r^\perp}X_j\|_2^2 \\
        &= \frac{\delta^2}{4}
        \|(I-U_rU_r^\top)X\|_F^2 
        =\frac{\delta^2}{4}
        \|(I-U_rU_r^\top)U\Sigma\|_F^2 \\
        &=\frac{\delta^2}{4}
        \|X-X_r\|_F^2 .
    \end{aligned}
    \]    
\end{proof}

The preceding channel-wise error bounds immediately imply layer-wise bounds by summing over the columns of $W$. We present both the unregularized least squares version and the regularized version, for completeness.

\begin{corollary}\label{cor:no_lam_layer}
    Assume an infinite alphabet $\mathcal{A}=\mathcal{A}^{\delta}= \{ \pm k\delta : k \in \mathbb{Z} \}$. If the unregularized version of \cref{lora_in_GPTQ} is implemented through the least squares formulation in \eqref{eq:q update} and \eqref{eq:w update}, then
    \[
        \|XW-X(Q+LR)\|_F^2
        \leq
        \frac{\delta^2 N'}{4}\|X-X_r\|_F^2 .
    \]
\end{corollary}

\begin{proof}
    Apply \cref{thm:no_lam_cha} to each column of $W$ and sum the resulting inequalities over all $N'$ columns.
\end{proof}

\begin{corollary}\label{thm:main_layer}
    Assume an infinite alphabet $\mathcal{A}=\mathcal{A}^{\delta}= \{ \pm k\delta : k \in \mathbb{Z} \}$. If \cref{lora_in_GPTQ} is run with regularization parameter $\lambda>0$, then
    \begin{equation}\label{eq:deterministic qlora layer bound}
        \|XW-X(Q+LR)\|_F^2
        +\lambda\|W-Q\|_F^2
        +\lambda\|LR\|_F^2
        \leq
        \frac{\delta^2 N'}{4}
        \left(\|X-X_r\|_F^2+(N+r)\lambda\right).
    \end{equation}
\end{corollary}

\begin{proof}
    Apply \cref{thm:main_channel} to each column $\mathbf{w}_i$ of $W$, with corresponding outputs $\mathbf{q}_i$ and $\mathbf{r}_i$. Summing over $i=1,\ldots,N'$ gives
    \[
        \|XW-X(Q+LR)\|_F^2
        +\lambda\|W-Q\|_F^2
        +\lambda\|R\|_F^2
        \leq
        \frac{\delta^2 N'}{4}
        \left(\|X-X_r\|_F^2+(N+r)\lambda\right).
    \]
    Since $L=V_r$ has orthonormal columns, $\|LR\|_F^2=\|R\|_F^2$. This gives the stated bound.
\end{proof}

\begin{remark}[Comparison with plain GPTQ]
    It is useful to compare \cref{eq:deterministic qlora layer bound} with the corresponding bound for plain GPTQ. In the notation of this section, \citet[Corollary 3.5]{zhang2025provable} gives
    \begin{equation}\label{eq:deterministic gptq bound}
        \|XW-XQ\|_F^2+\lambda \|W-Q\|_F^2
        \leq
        \frac{\delta^2 N'}{4}\left(\|X\|_F^2+N\lambda\right).
    \end{equation}
    The main improvement in \cref{eq:deterministic qlora layer bound} is that the term $\|X\|_F^2$ is replaced by the residual term $\|X-X_r\|_F^2$. Thus, when the calibration matrix $X$ admits a good rank $r$ approximation, the layer-wise reconstruction bound can be substantially smaller than the corresponding GPTQ bound. The price for this improvement is the additional regularization term $r\lambda$, which is comparatively small when $r\ll N$.
\end{remark}

We close this section on error bounds by demonstrating how the infinite alphabet bound in \cref{thm:main_layer} should be interpreted for finite precision quantization.

\begin{remark}[Finite alphabets and dynamic range]\label{finite bit remark}
    Choose $\lambda$ to be a sufficiently small constant multiple of $\|X-X_r\|_F^2/(N+r)$. Then the right hand side of \cref{eq:deterministic qlora layer bound} satisfies
    \[
        \frac{\delta^2 N'}{4}
        \left(\|X-X_r\|_F^2+(N+r)\lambda\right)
        \lesssim
        N'\delta^2\|X-X_r\|_F^2.
    \]
    Since every term on the left hand side of \cref{eq:deterministic qlora layer bound} is nonnegative, this gives
    \begin{equation}\label{eq:qlora error part}
        \|XW-X(Q+LR)\|_F^2
        \lesssim
        N'\delta^2\|X-X_r\|_F^2.
    \end{equation}
    Moreover, after dividing the regularization terms by $\lambda$, we obtain
    \begin{equation}\label{eq:qlora reg part}
        \|W-Q\|_F^2 \lesssim NN'\delta^2,
        \qquad
        \|LR\|_F^2 \lesssim NN'\delta^2,
    \end{equation}
    where we use that $r\leq N$ in the intended regime. 

    Thus \cref{thm:main_channel} and \cref{thm:main_layer} show that, on average, the entries of $Q$ deviate from the corresponding entries of $W$ by at most a constant multiple of $\delta$. If, in addition, one has the entrywise control $\|W-Q\|_\infty\lesssim \delta$, as one may expect generically, then a finite alphabet of the form
    \[
        \mathcal{A}_b^\delta
        =
        \{k\delta : k=-2^{b-1},\ldots,2^{b-1}\}
    \]
    suffices, provided $2^{b-1}\delta\geq \|W\|_\infty+C\delta$ for a suitable absolute constant $C$. The additive $C\delta$ term accounts for the extra dynamic range needed to absorb the adaptive GPTQ updates.

    This heuristic can be made rigorous by replacing deterministic rounding with stochastic rounding in \cref{lora_in_GPTQ}. Define the unbiased stochastic scalar quantizer $\mathcal{Q}_{\mathrm{stoc}}:\mathbb{R}\to \mathcal{A}^\delta$ as follows. If $z\in[k\delta,(k+1)\delta]$, then
    \[
        \mathcal{Q}_{\mathrm{stoc}}(z)
        :=
        \begin{cases}
            k\delta, & \text{with probability } p,\\
            (k+1)\delta, & \text{with probability } 1-p,
        \end{cases}
        \qquad
        p=1-\frac{z}{\delta}+k.
    \]
    Equivalently, $k=\lfloor z/\delta\rfloor$, and $\mathbb{E}\mathcal{Q}_{\mathrm{stoc}}(z)=z$.
    With a suitable choice of  $\lambda$, the arguments used in, for example, \citet[Theorem 4.6]{zhang2025provable} or \citet[Lemma 3.1]{zhang2023spfq}, imply that, with high probability,
    \[
        \|Q\|_\infty
        \leq
        \|W\|_\infty
        +
        K\delta\sqrt{\log(NN')},
    \]
    for a dimension independent constant $K$. We do not repeat the proof here. The corresponding version of \cref{thm:main_layer} has the same form, except that the factor $\delta^2/4$ on the right hand side is replaced by $\delta^2$.

    Now choose $\delta=c\|W\|_\infty/2^{b-1}$ for some constant $c>1$. Then
    \[
        \|Q\|_\infty
        \leq
        \delta
        \left(
        K\sqrt{\log(NN')}+\frac{2^{b-1}}{c}
        \right).
    \]
    Hence $\|Q\|_\infty\leq 2^{b-1}\delta$ whenever
    \[
        K\sqrt{\log(NN')}
        \leq
        \left(1-\frac{1}{c}\right)2^{b-1}.
    \]
    Thus, for fixed matrix dimensions and sufficiently large bit width $b$, the stochastic rounding implementation remains within the finite alphabet with high probability, with grid spacing on the natural scale $\delta\lesssim \|W\|_\infty/2^{b-1}$.
\end{remark}

\subsection{Near Optimality of GPTQ-intrinsic LoRA}\label{sec:bound is tight}

In this section, we compare the information theoretic lower bound in \cref{thm:data lower bound} with the upper bound achieved by GPTQ-intrinsic LoRA in \cref{thm:main_layer}. \SZnew{While \cref{thm:main_layer} is stated for an infinite alphabet, the discussions in \cref{finite bit remark} show that the bounds in \eqref{eq:qlora error part} and \eqref{eq:qlora reg part} continue to hold (with high probability) under the finite alphabet
\[
    \mathcal{A}_b^{\delta}
    =
    \{k\delta : k=-2^{b-1},\ldots,2^{b-1}\}
\]
with $\delta\lesssim \|W\|_\infty/2^{b-1}$ when using the stochastic rounding operator. In the lower bound, we use a slightly different parameterization of the quantization alphabet $\{-B,\ldots,-1,0,1,\ldots,B\}$, up to an adaptive scaling factor. 
To compare the two bounds, we set $B=2^{b-1}$ throughout this section. Thus both bounds share the same symmetric integer range $\{-2^{b-1},\ldots,2^{b-1}\}$ in representing the low-precision component.}

We assume that the calibration matrix has an approximate rank-$r$ structure. This is consistent with empirical observations that activations, and in some settings weights, of pretrained models often exhibit approximate low-rank structure \cite{huh2021low, yu2023compressing, zhang2025theoretical, zhang2024magr}. More precisely, let $X\in\R^{m\times N}$, with $m>N$, have singular values $\sigma_1\geq\cdots\geq\sigma_N>0$. We assume that the tail after the first $r$ singular values is relatively flat, in the sense that $\sigma_{r+1}\leq C\sigma_N$ for some absolute constant $C$, and that this tail is well separated from the leading singular values, so that $\sigma_{r+1}\ll \sigma_r$.

Let $W\in\R^{N\times N'}$ denote the weight matrix being approximated by a low-precision plus low-rank decomposition. Following the discussion in \cref{why flat}, we restrict attention to non-spiky weight matrices, and assume
\[
    W\in
    \mathbb{B}_{\mathrm{flat}}^1
    :=
    \left\{
    A\in\R^{N\times N'} :
    \|A\|_F\leq 1,\ 
    \|A\|_\infty\leq \frac{\sqrt{3}}{\sqrt{NN'}}
    \right\}.
\]

\subsubsection{Low-rank Corrections at the Quantization Error Scale}

We first consider the regime in which the low-rank correction has Frobenius norm on the same scale as the naive quantization error one would obtain if they were only using a round-to-nearest quantization algorithm. Thus we consider $\rho \asymp 1/B$ as this is the natural scale suggested by the flatness assumption on $W$. Indeed, since $\|W\|_\infty\leq \sqrt{3}/\sqrt{NN'}$, a natural grid spacing is $\delta \asymp 1/(B\sqrt{NN'})$. A naive round-to-nearest quantization of $W$ then has entrywise error at most $\delta/2$, and hence Frobenius error $\lesssim 1/B$. Thus, in this regime, the low-rank component is allowed to compensate for an error of the same order as the initial quantization error.

By \cref{thm:data lower bound}, there exists a matrix $W$ satisfying $\|W\|_F\leq 1$ and $\|W\|_\infty\leq \sqrt{3}/\sqrt{NN'}$ such that
\[
    \inf_{M\in\mathbb{S}_{B}^{r,\rho}\cap \mathbb{B}_F^2}
    \|X(M-W)\|_F^2
    \gtrsim
    \left(
    \frac{1}{(2B+1)^{\frac{J}{J-1}}
    \left(41\sqrt{\frac{\pi e}{6}}\rho\right)^{\frac{1}{J-1}}}
    \right)^2
    \sigma_N^2,
\]
where $J=J(r,N,N')=NN'/(r(N+N'+1)+2)\in(1,\infty)$. Now write $\rho=\tilde\rho/B$, where $\tilde\rho\asymp 1$. Since $J/(J-1)=1+1/(J-1)$, the denominator above is bounded above by
\(
    B\, C_1^{1/(J-1)}
\)
for a positive constant $C_1$ depending only on $\tilde\rho$ and absolute numerical constants. Consequently,
\[
    \inf_{M\in\mathbb{S}_{B}^{r,\rho}\cap \mathbb{B}_F^2}
    \|X(M-W)\|_F^2
    \gtrsim
    C_1^{-2/(J-1)}
    \frac{\sigma_N^2}{B^2}.
\]
In the simplifying case $N=N'$, one has
\[
    \frac{1}{J-1}
    =
    \frac{r(2N+1)+2}{N^2-r(2N+1)-2}.
\]
Thus, when $r\ll N$, the factor $C_1^{-2/(J-1)}$ is close to a constant and has first-order behavior governed by $r/N$ since 
\(
    C_1^{-2/(J-1)}
    =
    \exp\left(-\frac{2\log C_1}{J-1}\right)\approx 1-4\log C_1 \cdot \frac{r}{N}
\)
.

We now compare this lower bound with the upper bound achieved by \cref{lora_in_GPTQ}. Let $Q$ and $LR$ denote the low-precision and low-rank components returned by GPTQ-intrinsic LoRA using $X$ as the calibration matrix. By \cref{finite bit remark}, we may choose $\delta\lesssim \|W\|_\infty/B$, and therefore, under the flatness assumption on $W$,
\(
    \delta^2\lesssim \frac{1}{NN'B^2}.
\)
Moreover, the spectral assumption on $X$ gives
\[
    \|X-X_r\|_F^2
    =
    \sum_{j=r+1}^N \sigma_j^2
    \leq
    (N-r)\sigma_{r+1}^2
    \lesssim
    (N-r)\sigma_N^2.
\]
Substituting these estimates into \cref{thm:main_layer} gives
\[
    \|XW-X(Q+LR)\|_F^2
    \lesssim
    N'\delta^2\|X-X_r\|_F^2
    \lesssim
    \left(1-\frac{r}{N}\right)
    \frac{\sigma_N^2}{B^2}.
\]
The same finite-alphabet discussion also gives
\(
    \|LR\|_F^2\lesssim NN'\delta^2\lesssim \frac{1}{B^2},
  \) and 
    \( \|Q-W\|_F^2\lesssim NN'\delta^2\lesssim \frac{1}{B^2}.
\)
Thus $\|LR\|_F\lesssim 1/B$, so the low-rank factor satisfies the constraint $\|LR\|_F\leq \rho$, provided the implicit constant in $\rho\asymp 1/B$ is chosen large enough. Also, since $\|W\|_F\leq 1$, we have $\|Q\|_F-1\lesssim 1/B$, and hence $\|Q+LR\|_F\leq 2$ for sufficiently large $B$. Finally, the finite-alphabet construction ensures that $Q/\delta$ is integer-valued with entries bounded by $B$. Thus the output $Q+LR$ lies in the same constraint class $\mathbb{S}_{B}^{r,\rho}\cap\mathbb{B}_F^2$ used in the lower bound.

To summarize, in the regime $\rho\asymp 1/B$, the information theoretic lower bound gives
\[
    \inf_{M\in\mathbb{S}_{B}^{r,\rho}\cap \mathbb{B}_F^2}
    \|XW-XM\|_F^2
    \gtrsim
    C_1^{-2/(J-1)}
    \frac{\sigma_N^2}{B^2} \approx  \left(1-4\log C_1 \cdot \frac{r}{N}\right)\frac{\sigma_N^2}{B^2} ,
\]
while GPTQ-intrinsic LoRA achieves
\[
    \|XW-X(Q+LR)\|_F^2
    \lesssim
    \left(1-\frac{r}{N}\right)
    \frac{\sigma_N^2}{B^2}.
\]
Thus, for $r\ll N$, the upper bound matches the lower bound in its dependence on $B$ and $\sigma_N$, up to constants.

\subsubsection{Larger Low-rank Budgets and the Compression Tradeoff}

We now consider the more permissive regime $\rho\asymp 1$. The purpose of this subsection is to explain what changes when the low rank component is allowed to have substantially larger Frobenius norm than the naive quantization error that would be obtained via round-to-nearest quantization. The main conclusion is that the lower bound is weakened only through a $J$-dependent factor. Making this factor very small requires $J$ to be close to $1$, which corresponds to using a low rank component with enough parameters to compromise the intended compression.

By \cref{thm:data lower bound}, there exists a matrix $W$ satisfying $\|W\|_F\leq 1$ and $\|W\|_\infty\leq \sqrt{3}/\sqrt{NN'}$ such that
\[
    \inf_{M\in\mathbb{S}_{B}^{r,\rho}\cap \mathbb{B}_F^2}
    \|X(M-W)\|_F^2
    \gtrsim
    \frac{\sigma_N^2}
    {(2B+1)^{\frac{2J}{J-1}}
    \left(41\sqrt{\frac{\pi e}{6}}\rho\right)^{\frac{2}{J-1}}},
\]
where $J=J(r,N,N')=NN'/(r(N+N'+1)+2)\in(1,\infty)$. Since $J/(J-1)=1+1/(J-1)$, this can be rewritten, up to absolute constants, as
\[
    \inf_{M\in\mathbb{S}_{B}^{r,\rho}\cap \mathbb{B}_F^2}
    \|X(M-W)\|_F^2
    \gtrsim
    \frac{\sigma_N^2}{B^2}
    (C_2B)^{-\frac{2}{J-1}},
\]
where $C_2$ is proportional to $\rho$, up to absolute numerical constants.

This expression separates the dominant $B^{-2}$ quantization scale from the additional factor $(C_2B)^{-2/(J-1)}$ that reflects the larger low rank magnitude budget. To interpret this factor, suppose for simplicity that $N=N'$. Then, as we already calculated in the previous subsection,
\(
    \frac{1}{J-1}
    =
    \frac{r(2N+1)+2}{N^2-r(2N+1)-2}.
\)
When $r\ll N$, this exponent is small, with leading scale $2r/N$. Hence, up to a first order approximation
\[
    (C_2B)^{-\frac{2}{J-1}}
    =
    \exp\left(-\frac{2\log(C_2B)}{J-1}\right)
\approx   1 -4\log(C_2B)\frac{r}{N}.
\]
Thus, in the small rank regime, increasing the allowed low rank magnitude weakens the lower bound through a relatively mild factor, while the dominant dependence on the alphabet size remains $B^{-2}$.

On the upper bound side, allowing $\rho\asymp 1$ only makes the admissible class larger. The output of GPTQ-intrinsic LoRA already satisfies $\|LR\|_F\lesssim 1/B$ under the finite alphabet scaling discussed in \cref{finite bit remark}. Therefore the approximation produced in the previous subsection remains admissible when $\rho\asymp 1$, and the same upper bound continues to hold:
\[
    \|XW-X(Q+LR)\|_F^2
    \lesssim
    \left(1-\frac{r}{N}\right)\frac{\sigma_N^2}{B^2}.
\]
Consequently, in the regime $r\ll N$, the upper bound still matches the lower bound in its dominant dependence on $B$ and $\sigma_N$, with the gap appearing in the mild factor $(C_2B)^{-2/(J-1)}$.

Finally, the exact lower bound suggests that one can make the lower bound smaller by making $J$ close to $1$. This is not a contradiction to the compression interpretation. The parameter $J$ is essentially the compression ratio associated with the low rank component, since the low rank factors use roughly $1/J$ times as many parameters as the original matrix. For example, if the low rank factors are stored in $16$-bit precision and the quantized component $Q$ uses $b$ bits per entry, then the average number of bits per original weight is approximately
\(
    b+\frac{16}{J}.
\)
For the compressed representation to use fewer than $16$ bits per original weight, one needs
\(
    b+\frac{16}{J}<16,
\)
or equivalently
\(
    b\frac{J}{J-1}<16.
\)
Thus, taking $J$ too close to $1$ may improve the approximation power of the low rank component, but it also destroys the memory savings.

\section{Bid-Up: Alternating Refinement of Quantization and Low-Rank Compensation}\label{sec:bid-up}

Let $Q^{(0)},L^{(0)},R^{(0)}$ denote the output of GPTQ-intrinsic LoRA in \cref{lora_in_GPTQ}. We now use this output as a principled initialization for further refinement. The idea is that with $Q$ fixed, we improve the low-rank component using OLrC ((\cref{OLrC})). With $L,R$ fixed, we improve the quantized component using a new fixed-grid coordinate update, which we call Bid-Up. Each step is designed so that the layer-wise reconstruction error can decrease, and certainly does not increase.

In more detail, fixing $Q^{(0)}$ we update $L^{(0)},R^{(0)}$ to $L^{(1)},R^{(1)}$ using OLrC in \cref{OLrC}. Since OLrC solves the optimal rank-$r$ compensation problem for a fixed quantized matrix, this update can only improve, or leave unchanged, the layer-wise reconstruction error. It is then natural to alternate between improving the low-rank component and improving the quantized component.

A naive update of the quantized component, however, is not guaranteed to improve the same objective. For example, after updating $L^{(t)},R^{(t)}$ to $L^{(t+1)},R^{(t+1)}$, one could try to re-quantize $W-L^{(t+1)}R^{(t+1)}$ from scratch using GPTQ with the same calibration matrix $X$. In practice, this may not yield an error decrease. One reason is that re-quantizing from scratch typically recomputes the quantization parameters from the per-channel statistics of the new residual matrix $W-L^{(t+1)}R^{(t+1)}$. Its dynamic range may be larger than that of the previous residual, and the new quantization grid may not be comparable to the old one. Thus the layer-wise reconstruction objective in \cref{eq:layer_err} is not guaranteed to decrease. This difficulty is consistent with the observation in \citet{saha2024compressing}, where many iterations of alternating minimization were used to obtain improvement, while several subsequent methods apply OLrC only once after fixing $Q$ \cite{deng2025cloq, zhang2024qera, liu2024eora}.

Bid-Up addresses this issue by updating $Q$ while keeping the quantization grid fixed. Given the new low-rank factors $L^{(t+1)},R^{(t+1)}$, set
\[
    \widetilde W = W-L^{(t+1)}R^{(t+1)}.
\]
The goal is to replace $Q^{(t)}$ by a refined quantized matrix $Q^{(t+1)}$ satisfying
\[
    \|X\widetilde W-XQ^{(t+1)}\|_F^2
    \leq
    \|X\widetilde W-XQ^{(t)}\|_F^2.
\]
As before, since
\(    \|X\widetilde W-XQ\|_F^2
    =
    \sum_{j=1}^{N'}\|X\widetilde W_j-XQ_j\|_2^2,
\)
the columns can be processed independently and in parallel.

Fix one column $\tilde w$ of $\widetilde W$, and let $q^{(t)}$ be the corresponding column of $Q^{(t)}$. Bid-Up updates the entries of $q$ sequentially. Suppose $q_1^{(t+1)},\ldots,q_{i-1}^{(t+1)}$ have already been updated, while $q_{i+1}^{(t)},\ldots,q_N^{(t)}$ are still kept at their old values. Then the next entry is chosen by
\[
    q_i^{(t+1)}
    =
    \arg\min_{p\in\mathcal A}
    \left\|
    X\tilde w
    -
    X_{\leq i-1}q_{\leq i-1}^{(t+1)}
    -
    X_i p
    -
    X_{\geq i+1}q_{\geq i+1}^{(t)}
    \right\|_2^2.
\]
As this is a one-dimensional least squares problem over the fixed alphabet $\mathcal A$, it has the closed form
\[
    q_i^{(t+1)}
    =
    \mathcal Q\left(
    \frac{
    \left\langle
    X_i,\,
    X\tilde w
    -
    X_{\leq i-1}q_{\leq i-1}^{(t+1)}
    -
    X_{\geq i+1}q_{\geq i+1}^{(t)}
    \right\rangle
    }{
    \|X_i\|_2^2
    }
    \right).
\]
This update chooses the best value of the $i$th quantized coordinate, given the current values of all other coordinates. Therefore each coordinate update can only decrease, or leave unchanged, the current reconstruction error.
The memory-efficient implementation in \cref{Bid-up} performs this coordinate update for all columns simultaneously.

\begin{algorithm}[t]
\caption{Bid-Up}
\label{Bid-up}
\begin{algorithmic}[1]
\STATE \textbf{Input:} $H=X^\top X$, residual weight matrix $\widetilde W$, current quantized matrix $Q_{\mathrm{initial}}$, quantizer $\mathcal Q$
\STATE \textbf{Output:} updated quantized matrix $Q$
\STATE $D=\operatorname{diag}(H)$, $C=H-D$
\FOR{each column $\widetilde{w}$ of $\widetilde{W}$, in parallel}
    \STATE $q = q_{\mathrm{initial}}$
    \FOR{$i=1,\ldots,N$}
        \STATE $q_{i}=\mathcal{Q}\bigl(\frac{1}{D_{ii}}\langle H_i,\widetilde{w}  \rangle - \frac{1}{D_{ii}}\langle C_i,q  \rangle \bigr)$
    \ENDFOR
\ENDFOR
\STATE \textbf{return} $Q$
\end{algorithmic}
\end{algorithm}

The in-place implementation is important. When row $i$ is updated, rows $1,\ldots,i-1$ of $Q$ already contain their new values, while rows $i+1,\ldots,N$ still contain their old values. Thus the matrix formula in \cref{Bid-up} implements the coordinatewise update above for all columns simultaneously.

Because the quantization grid is fixed during Bid-Up, every coordinate update is an exact minimization over the same finite alphabet. Hence
\(
    \|X(W-L^{(t+1)}R^{(t+1)})-XQ^{(t+1)}\|_F^2
    \leq
    \|X(W-L^{(t+1)}R^{(t+1)})-XQ^{(t)}\|_F^2.
\)
The OLrC step gives the complementary improvement
\(
    \|X(W-L^{(t+1)}R^{(t+1)})-XQ^{(t)}\|_F^2
    \leq
    \|X(W-L^{(t)}R^{(t)})-XQ^{(t)}\|_F^2.
\)
Combining the two inequalities, one full refinement cycle satisfies
\[
    \|X(W-L^{(t+1)}R^{(t+1)})-XQ^{(t+1)}\|_F^2
    \leq
    \|X(W-L^{(t)}R^{(t)})-XQ^{(t)}\|_F^2.
\]
Thus GPTQ-intrinsic LoRA provides a provably controlled initialization through \cref{thm:main_layer}, and the subsequent OLrC and Bid-Up steps refine the low-rank and quantized components while ensuring that the layer-wise reconstruction error {will decrease until convergence}. Bid-Up is also not tied to GPTQ-intrinsic LoRA. It can be used as a fixed-grid quantization refinement step after any PTQ method.

\section{Experiments}

In this section, we present numerical results to demonstrate the effectiveness of our proposed GPTQ-intrinsic LoRA (\cref{lora_in_GPTQ}), and the gains from additional loops of OLrC (\cref{OLrC}) and Bid-Up (\cref{Bid-up}). All results are based on per-channel weight quantization with the standard asymmetric weight quantizer. Experiments are implemented in PyTorch \citep{paszke2019pytorch} using a single Nvidia A100 GPU. 

We first present numerical results on language models in \cref{tbl:qwen3} with $r=16, 32, 64$. We experiment with Qwen3 \citep{yang2025qwen3} models using WikiText2 datasets \citep{merity2016pointer} for evaluation. We use, without modification, the implementations made publicly available via Huggingface \citep{wolf2020transformers}. The Qwen3-1.7B and 0.6B models utilize a standard decoder-only transformer architecture comprising 28 blocks, each containing 7 fundamental linear layers distributed across the grouped-query attention (Q, K, V, O) and SwiGLU-based feed-forward network (Gate, Up, Down) modules. The main difference between the 2 models lies in their embedding dimension, which is 1024 for  the 0.6B model, and 2048 for the 1.7B model. We construct our calibration dataset using 128 random sequences of 2048 tokens sampled from the WikiText2 training dataset for all data-driven algorithms. Following \citet{zhang2025qronos}, the 128 sequences are sampled once at the beginning and used to generate the input calibration matrix $X$ (or equivalently its associated Hessian $H$) for all the 196 linear layers. We use the entire testing set to evaluate the perplexity. Our 4-bit results use the min-max grid, \textit{i.e. $\beta=1$}, and the 3-bit results use a scaled min-max grid with $\beta=0.9$. More implementation details can be found in \cref{sec:better cholesky}. 

Following the discussion in \cref{sec:background}, we choose GPTQ (\cref{GPTQ}) followed by OLrC to be the baseline method to compare against as this approach is representative of multiple recent works \cite{liu2024eora, zhang2024qera, deng2025cloq, saha2024compressing}. We include the plain GPTQ results in \cref{tbl:qwen3}. When reproducing GPTQ, we apply standard dampening by replacing the Hessian $H=X^\top X$ with $H+\lambda I$, where $\lambda=0.01\,\operatorname{mean}(\operatorname{diag}(H))$. We follow the original approach of the publicly available code in \cite{ist2022gptq} to obtain the lower-triangular Cholesky factor $\Psi$ of $(H+\lambda I)^{-1}$. To be specific, the original implementation first computes one Cholesky decomposition $H+\lambda I = \Phi\Phi^\top$, then applies a Cholesky inverse routine to recover $(H+\lambda I)^{-1}$ from $\Phi$, and finally computes a second Cholesky factorization $(H+\lambda I)^{-1}=\Psi\Psi^\top$ to get $\Psi$. This approach avoids explicitly forming the inverse through general-purpose matrix inversion routines and instead exploits the positive definiteness of $H+\lambda I$. In particular, Cholesky-based inversion inherits the numerical advantages of triangular solvers and reduces error amplification compared with explicit general-purpose matrix inversion. GPTQ* denotes our modified (more robust) implementation of GPTQ using the eigen-decomposition plus QR decomposition approach in \cref{lora_in_GPTQ} to compute the Cholesky factor, as discussed in more details in \cref{sec:better cholesky}. We can observe that our robust implementation of GPTQ noticeably improves the results for GPTQ for both model sizes under 3-bit and 4-bit quantization.

In the 3-bit experiments, the row labeled ``With $1$ refinement loop'' corresponds to applying one additional OLrC update followed by one Bid-Up update after GPTQ-intrinsic LoRA. The results show that this refinement step consistently improves performance across the tested settings, despite the fact that GPTQ-intrinsic LoRA already substantially outperforms the GPTQ/GPTQ*+OLrC baselines.

\begin{table}[ht!]
\centering
\caption{\textbf{\textbf{Wiki2 PPL}($\downarrow$) results on weight-only quantization of Qwen3 models.}}
\resizebox{0.9\textwidth}{!}
{\begin{tabular}{clccc!{\vrule width 0.5pt}ccc}
\toprule
 &  & \multicolumn{3}{c}{0.6B} & \multicolumn{3}{c}{1.7B} \\
 &  & $r=16$ & $r=32$ & $r=64$ & $r=16$ & $r=32$ & $r=64$ \\ 
 \midrule
BF16 & ----- & &18.83& & &15.64& \\  
\midrule
\multirow{3}{*}{4-bit} & GPTQ/GPTQ* & &24.33 /23.98& & &25.90 /22.42& \\  
& GPTQ/GPTQ*+OLrC  & 22.43 /22.51  & 21.84 /22.08  & 21.13 /21.14  & 19.92 /19.59  & 19.22 /17.52  & 16.72 /17.87   \\
& GPTQ-intrinsic LoRA & {\bf 20.77} & {\bf 20.18} & {\bf 20.10} & {\bf 16.88} & {\bf 16.64} & {\bf 15.61} \\
\midrule
\multirow{3}{*}{3-bit} & GPTQ/GPTQ* & &49.01 /47.67& & &50.48 /45.10& \\  
& GPTQ/GPTQ*+OLrC  & 46.25 /47.37  & 39.35 /37.65  & 34.85 /35.24  & 34.78 /30.62  & 31.74 /26.87 &  27.71 /24.06 \\
& GPTQ-intrinsic LoRA & 29.66 & 28.94 & 26.67 & 25.64 & 22.32 & 20.49 \\
& With 1 Refinement Loop  & {\bf 28.02}  & {\bf 27.83}  & {\bf 25.24}  & {\bf 24.12}  & {\bf 21.43}  & {\bf 19.86}  \\
\bottomrule
\end{tabular}}
\label{tbl:qwen3}
\end{table}

Next, we present ablation results for vision transformers on the ImageNet classification benchmark. We use ILSVRC-2012 \cite{deng2009imagenet}, a 1000-class dataset with 1.28 million training images and 50 thousand validation images. Images are preprocessed in the standard way: each image is resized to $256\times 256$, followed by a normalized $224\times 224$ center crop. Both models are loaded from the Hugging Face timm library \cite{wightman2021resnet} using the \texttt{patch16\_224} variants.

We evaluate the top-1 accuracy of the original and quantized models on the full validation set. For all data-driven algorithms, we use batches of size $2048$ to generate calibration data. Following \citet{zhang2023post}, we use a new batch of 2048 images to form the calibration matrix $X$ for each layer. Within a given layer, the same $X$ is used for GPTQ-intrinsic LoRA and for all subsequent OLrC and Bid-Up refinement loops.

\cref{tab1} reports results for DeiT-B \cite{touvron2021training}, which has 86 million parameters. The original top-1 accuracy is $81.74\%$. For this model, we use the per-channel scaling parameter $\beta=0.6$. Under 2-bit quantization, the accuracy drops to $74.54\%$ with GPTQ and $74.55\%$ with GPTQ*. Similarly, \cref{tab2} reports results for DeiT-III-L \cite{touvron2022deit}, which has 304 million parameters. The original top-1 accuracy is $84.59\%$. For this model, we use $\beta=0.7$. Under 2-bit quantization, the accuracy drops to $80.73\%$ with GPTQ and $80.70\%$ with GPTQ*.

For these vision transformer experiments, GPTQ* gives nearly the same results as GPTQ. This is consistent with the empirical observation that the calibration matrices $X$ arising in vision transformers are much less ill-conditioned than those arising in language models. Better conditioning improves the numerical stability of Cholesky-based quantization algorithms, but it can also leave less room for improvement from the more robust Cholesky computation used in GPTQ*. Nevertheless, \cref{tab1} and \cref{tab2} show that GPTQ-intrinsic LoRA still recovers a substantial portion of the lost accuracy, and that additional OLrC and Bid-Up refinement loops can further improve performance.

\begin{table}[htbp]
\centering
\caption{Top-1 accuracy ($\uparrow$) for 2-bit weight-only quantization of DeiT-B. Here ``$+k$ loops" denotes $k$ additional refinement loops, each consisting of one OLrC update followed by one Bid-Up update.}
\begin{tabular}{c|ccccc}
\toprule
 & $b_L=4,\ b_R=8$ & GPTQ-intrinsic LoRA & $+1$ loop & $+2$ loops & $+3$ loops \\
\midrule
$r=32$ & 77.50 & 77.45 & 77.95 & 78.06 & \textbf{78.07} \\
$r=16$ & 76.97 & 76.76 & 77.24 & \textbf{77.48} & 77.42 \\
$r=8$  & 76.05 & 76.26 & 76.95 & 76.98 & \textbf{77.20} \\
\bottomrule
\end{tabular}
\label{tab1}
\end{table}

The column labeled $b_L=4,\ b_R=8$ in both tables reports results in which the low-rank factors are themselves quantized. Specifically, the left factor $L$, which serves as the feature extraction factor, is quantized to 4 bits with per-column scaling, while the right factor $R$, which serves as the error absorption factor, is quantized to 8 bits with per-row scaling. The corresponding scaling factors can be combined across the rank dimension, so only $r$ additional full-precision scaling factors need to be stored, which is negligible relative to the model size. In these experiments, $L=V_r$ is quantized by round-to-nearest before forming the augmented Hessian in \cref{lora_in_GPTQ}. This allows GPTQ-intrinsic LoRA to absorb the rounding error introduced by quantizing $L$. The factor $R$, on the other hand, is first computed in full precision by \cref{lora_in_GPTQ} and then quantized by round-to-nearest. The results show that quantizing the low-rank factors from FP32 to low precision is often essentially lossless, and in several cases even improves accuracy.

\begin{table}[htbp]
\centering
\caption{Top-1 accuracy ($\uparrow$) for 2-bit weight-only quantization of DeiT-III-L. Here ``$+k$ loops'' denotes $k$ additional refinement loops, each consisting of one OLrC update followed by one Bid-Up update.}
\begin{tabular}{c|ccccc}
\toprule
 & $b_L=4,\ b_R=8$ & GPTQ-intrinsic LoRA & $+1$ loop & $+2$ loops & $+3$ loops \\
\midrule
$r=32$ & 82.34 & 82.23 & 82.54 & 82.45 & \textbf{82.60} \\
$r=16$ & 82.07 & 82.04 & 82.25 & 82.23 & \textbf{82.44} \\
$r=8$  & 81.72 & 81.67 & 82.06 & 82.08 & \textbf{82.10} \\
\bottomrule
\end{tabular}
\label{tab2}
\end{table}

\section{Conclusions and Future Work}

In this paper, we establish, to our knowledge, the first information-theoretic lower bounds on the layer-wise reconstruction error for low-precision quantization with low-rank compensation. We also propose GPTQ-intrinsic LoRA, an efficient training-free algorithm that uses only a small calibration dataset, and that matches these lower bounds in scaling up to constants and mild factors. We also propose Bid-Up, a fixed-grid quantization refinement technique that can be combined with the low-rank compensation method OLrC to further refine both the low-precision and low-rank components produced by GPTQ-intrinsic LoRA. Our numerical results demonstrate the empirical effectiveness of GPTQ-intrinsic LoRA, as well as the additional gains obtained from refinement loops combining OLrC and Bid-Up.

None of the methods proposed in this paper requires training with labeled data. A natural next step is to use the output of GPTQ-intrinsic LoRA, possibly after additional OLrC and Bid-Up refinement loops, as an initialization for fine tuning the low-rank adapters while keeping the quantized weights fixed \cite{dettmers2024qlora, deng2025cloq}. We leave this direction to future work in order to preserve the concise and theory-focused scope of the present paper, and also because of our limited computational resources.

The near-optimality of GPTQ-intrinsic LoRA may also inform scaling laws for compressed representations. Existing work has primarily focused on unified scaling laws for sparsity and low-precision quantization \cite{frantar2025compression, panferov2026unified}. Extending such results to include low-rank representations remains an open problem. In particular, it would be interesting to develop a unified rate-distortion framework, based on effective bit-width, that can characterize representations combining low-precision, low-rank, and sparse components. A complementary direction is to develop improved budget allocation strategies for bit-width and rank across layers \cite{zhou2026autoqra}, rather than using uniform allocations throughout the network, in order to more closely approach the rate-distortion Pareto frontier.

Other possible directions include incorporating preprocessing rotations \cite{ashkboos2024quarot, chee2023quip, liu2024spinquant} into GPTQ-intrinsic LoRA, extending the framework to activation quantization \cite{xiao2023smoothquant} and KV cache quantization \cite{liu2024kivi, wang2026qsvd}, and exploring the interplay with microscaling under different quantization data types \cite{chen2025int, cook2025four, egiazarian2025bridging}.
\section*{Acknowledgments}
We gratefully acknowledge partial support by the National Science Foundation, via the DMS-2410717 grant.

\newpage
\bibliographystyle{abbrvnat}
\bibliography{citations}
\newpage
\appendix

\section{Proofs}\label{sec:proofs}

\subsection{Proof of Theorem \ref{thm:density}}
\begin{proof}

By transposing if necessary, it suffices to consider the case $N \geq N'$, since the operator norm is invariant under transposition and the transpose of a rank-one matrix is rank one. If $N'=1$, the claim is immediate, so we assume $N'\ge 2$.

It also suffices to prove the result for $\delta=1$. Indeed, let $\epsilon>0$ be arbitrary. Applying the $\delta=1$ case to $W/\delta$ with tolerance $\epsilon/\delta$ gives $Z\in \mathbb{Z}^{N\times N'}$ and a rank-one correction $ab^\top$ such that
\(
    \left\| \frac{W}{\delta} - (Z+ab^\top)\right\|_{\mathrm{op}} < \frac{\epsilon}{\delta}.
\)
Multiplying by $\delta$ gives
\(
    \|W - (\delta Z + (\delta a)b^\top)\|_{\mathrm{op}} < \epsilon.
\)
Thus, from now on, assume $\delta=1$.

    Choose irrational numbers \(\alpha_1, ..., \alpha_{N'-1} \)  such that the set \(\{1, \alpha_1, ..., \alpha_{N'-1} \}\) is linearly independent over \(\mathbb{Q}\). For example one may pick \(\alpha_i = \sqrt{p_i}\) where \(p_i\) is the \(i\)-th prime number. 
    For \(1 \le i \le N'-1\), define \(u_i = e_i + \alpha_i e_{N'} \in \mathbb{R}^{N'}\) where \(e_i\) denotes the \(i\)-th standard basis vector in $\mathbb{R}^{N'}$. Let \(U \in \mathbb{R}^{N' \times N'-1}\) be a matrix whose columns are the vectors \(u_i\).

    We claim that
    \[\{U^\top c \ |\ c \in \mathbb{Z}^{N'}\} \subseteq \mathbb{R}^{N'-1}\]
    is dense in \(\mathbb{R}^{N'-1}\). Indeed, $\langle u_i, c \rangle = \langle e_i, c \rangle + \alpha_i\langle e_N, c \rangle = c_i + \alpha_i c_N \equiv \alpha_i c_N  \mod \mathbbm{Z}$. Since the integers $c_i, 1 \le i \le N'-1$ may be chosen freely, it suffices to show that 
    \[\left\{c_N \begin{pmatrix}
        \alpha_1 \\
        \vdots \\
        \alpha_{N'-1}
    \end{pmatrix} \mod \mathbb{Z} \ |\ c_N \in \mathbb{Z}\right\} \subseteq \mathbb{T}^{N'-1}\]
    is  dense in the Torus $\mathbb{T}^{N'-1}$. This follows from the (discrete) Kronecker-Weyl equidistribution theorem \cite{bailleul2022explicit} because \(1,\alpha_1,...,\alpha_{N'-1}\) are linearly independent over $\mathbb{Q}$.

    Now, take the \(QR\)-decomposition \(U = VR\) where \(V\in\R^{N'\times N'-1}\) has orthonomal columns and \(R\) is invertible. Since
    \(
    U^\top c = R^\top V^\top c,
\)
and $R^\top$ is invertible, the set
\(
    S:=\{V^\top c : c\in \mathbb{Z}^{N'}\}
\)
is also dense in $\mathbb{R}^{N'-1}$.

    We proceed to approximate each row of \(WV \in \mathbb{R}^{N \times N'-1}\) by elements of \(S\). Indeed,  since \(S\) is dense, there exists \(z_1, ..., z_{N} \in \mathbb{Z}^{N'}\) such that  \[Z := \begin{bmatrix}
        z_1^\top \\
        \vdots \\
        z_{N}^\top
    \end{bmatrix} \in \mathbb{Z}^{N \times N'}\]  satisfies
    $\|Z V - WV\|_F=\|(Z-W)V\|_F < \epsilon.$
    Let $v_1,\ldots,v_{N'-1}$ denote the columns of $V$. Then it follows that for any \(v \in \mathrm{span} \{v_1, ..., v_{N'-1}\}
    \subseteq\R^{N'}\) with \(\|v\|_2 = 1\), we have
    \(\|(Q - W)v\|_2 < \epsilon.\) 

    Finally, we let \(T = \mathrm{span} \{v_1, ..., v_{N'-1})\) and denote \(\sigma_i(Q-W)\) to be the \(i\)th largest singular value of \(Q-W\). By the Courant-Fischer Theorem\cite{ikebe1987monotonicity}
    \[\sigma_2(Q-W) = \min_{\dim E = N'-1} \|(Q-W)|_E\|_{op} \leq \|(Q-W)|_T\|_{op} < \epsilon.\]

    Therefore, picking the best rank-one approximation \(a b^\top\) to \(W-Q\) yields
    \[\| W- (Q + ab^\top)\|_{op} < \epsilon.\]
    As $\epsilon>0$ was arbitrary, the proof is complete.
\end{proof}

\subsection{Proof of Lemma \ref{lemma:cover lr mat}}
\begin{proof}

    It is well known \cite{candes2011tight} that the set of matrices in $\mathbb{R}^{N\times N'}$ with Frobenius norm $1$ and rank at most $r$ admits an $\epsilon$-net in Frobenius norm of cardinality at most
\(
    \left(\frac{9}{\epsilon}\right)^{r(N+N'+1)}.
\)

To form an $\epsilon$-net of the one-dimensional segment $[0,\rho]$, take points in $[0,\rho]$ with spacing at most $2\epsilon$. Then every point of $[0,\rho]$ is within distance $\epsilon$ of one of these points, provided the endpoint gaps are also at most $\epsilon$. Hence one may choose such a net with cardinality at most
\(
    \left\lceil \frac{\rho}{2\epsilon}\right\rceil+1.
\)
In particular, since $\epsilon\le \rho$, this cardinality is bounded by $2\rho/\epsilon$ as both \(\left\lceil \frac{\rho}{2\epsilon}\right\rceil\) and $1$ are bounded by \(\rho/\epsilon\).

    For any $M\in\mathbb{S}^{r,\rho}$, write $M=\rho^\star M_0$, where $0\le \rho^\star\le \rho$ and $M_0$ has Frobenius norm $1$ and rank at most $r$, with $M_0$ chosen arbitrarily if $M=0$. Set $\epsilon_1=\epsilon/(1.1\rho)$, and let $S_1$ be an $\epsilon_1$-net for the unit Frobenius norm rank-at-most-$r$ matrices. Since $\epsilon\le \rho$, we have $\epsilon_1<1$, and
\[
    |S_1|
    \le
    \left(\frac{9.9\rho}{\epsilon}\right)^{r(N+N'+1)}
    \le
    \left(\frac{10\rho}{\epsilon}\right)^{r(N+N'+1)}.
\]
    Let $S_2$ be an $\epsilon_2$-net for $[0,\rho]$, with $\epsilon_2$ to be chosen. Taking points with spacing $2\epsilon_2$ gives
\[
    |S_2|\le \left\lceil \frac{\rho}{2\epsilon_2}\right\rceil+1.
\]
Choose $\widetilde M_0\in S_1$ and $\widetilde \rho\in S_2$ such that
\[
    \|M_0-\widetilde M_0\|_F\le \epsilon_1,
    \qquad
    |\rho^\star-\widetilde \rho|\le \epsilon_2.
\]

    Thus,\[
    \|\rho^\star M_0-\tilde{\rho}\widetilde{M}_0\|_F\leq \|(\rho^\star-\tilde{\rho})M_0\|_F+\|\tilde{\rho}(M_0-\widetilde{M}_0)\|_F\leq \epsilon_2 + \rho\epsilon_1=\epsilon_2+\frac{\epsilon}{1.1}.
    \]
    Choose $\epsilon_2=\epsilon-\frac{\epsilon}{1.1}=\frac{\epsilon}{11}$ so the right-hand side above is $\epsilon$. Moreover,  $|S_2|\leq \lceil\frac{5.5\rho}{\epsilon}\rceil +1\leq \frac{5.5\rho}{\epsilon} +2 < \frac{5.5\rho}{\epsilon} +\frac{2\rho}{\epsilon} \leq \frac{10\rho}{\epsilon}$. Hence the products $\widetilde\rho\,\widetilde M_0$, with $\widetilde\rho\in S_2$ and $\widetilde M_0\in S_1$, form an $\epsilon$-net for $\mathbb{S}^{r,\rho}$, with
cardinality bounded by
  \(  
    \left(\frac{10\rho}{\epsilon}\right)^{r(N+N'+1)+1}.
\)
\end{proof}

\subsection{Proof of Corollary \ref{thm:fancy lb corollary}}

\begin{proof}
Let $d=NN'$. We have
\begin{align*}
    \frac{\operatorname{Vol}(\mathbb{B}_{\mathrm{flat}}^1)}
    {\operatorname{Vol}(\mathbb{B}_F^1)}
    &=
    \frac{
    \operatorname{Vol}\bigl(
    \{x\in\mathbb{R}^d:\|x\|\le 1,\ \|x\|_\infty\le \sqrt{3}/\sqrt d\}
    \bigr)}
    {
    \operatorname{Vol}\bigl(
    \{x\in\mathbb{R}^d:\|x\|\le 1\}
    \bigr)}
    \\
    &=
    \frac{
    \operatorname{Vol}\bigl(
    \{x\in\mathbb{R}^d:\|x\|\le \sqrt d,\ \|x\|_\infty\le \sqrt 3\}
    \bigr)}
    {
    \operatorname{Vol}\bigl(
    \{x\in\mathbb{R}^d:\|x\|\le \sqrt d\}
    \bigr)}
    \\
    &\ge
    \frac{
    \operatorname{Vol}\bigl(
    \{x\in\mathbb{R}^d:\|x\|\le \sqrt d,\ \|x\|_\infty\le \sqrt 3\}
    \bigr)}
    {\frac{1.1}{\sqrt{\pi d}}(\sqrt{2\pi e})^d},
\end{align*}
where the last step uses Stirling's approximation for the volume of a Euclidean ball. We now bound the numerator. Let $X_1,\dots,X_d$ be i.i.d. $\mathrm{Unif}([-1,1])$ random variables. Set $S_d=\sum_{i=1}^d X_i^2$ and $Y=X^2$, where $X\sim\mathrm{Unif}([-1,1])$. A direct computation gives $\mu=\mathbb{E}Y=1/3$, $\sigma^2=\operatorname{Var}(Y)=1/5-1/9=4/45$, and
\[
    \frac{\mathbb{E}|Y-\mu|^3}{\sigma^3}<5.4.
\]
Then
\begin{align*}
    \operatorname{Vol}\bigl(
    \{x\in\mathbb{R}^d:\|x\|\le \sqrt d,\ \|x\|_\infty\le \sqrt 3\}
    \bigr)  &=
    (2\sqrt 3)^d
    \mathbb{P}\left(3\sum_{i=1}^d X_i^2\le d\right) \\
    & =
    (2\sqrt 3)^d
    \mathbb{P}\left(\frac{S_d}{\sqrt d}\le \frac{\sqrt d}{3}\right) \\
    & =
    (2\sqrt 3)^d
    \mathbb{P}\left(\frac{S_d-\mu d}{\sigma\sqrt d}\le 0\right).
\end{align*}
By Berry-Esseen, see for example \cite{shevtsova2007sharpening},
\[
\left|
\mathbb{P}\left(\frac{S_d-\mu d}{\sigma\sqrt d}\le 0\right)
-\mathbb{P}(Z\le 0)
\right|
\le
0.7655\frac{\mathbb{E}|Y-\mu|^3}{\sigma^3}\frac{1}{\sqrt d}
<
\frac{5}{\sqrt d},
\]
where $Z\sim N(0,1)$. Thus
\begin{equation}
\begin{aligned}
    \frac{\operatorname{Vol}(\mathbb{B}_{\mathrm{flat}}^1)}
    {\operatorname{Vol}(\mathbb{B}_F^1)}
    &\ge
    \frac{(2\sqrt 3)^d(1/2-5/\sqrt d)}
    {\frac{1.1}{\sqrt{\pi d}}(\sqrt{2\pi e})^d}
    \\
    &=
    \frac{1}{1.1}
    \left(\frac{1}{2}\sqrt{\pi d}-5\sqrt{\pi}\right)
    \left(\frac{2\sqrt 3}{\sqrt{2\pi e}}\right)^d
    \\
    &>
    \left(\frac{2\sqrt 3}{\sqrt{2\pi e}}\right)^d,
\end{aligned}
\end{equation}
where $d\ge 144$ is used in the last step.
As in the proof of \cref{thm:fancy lower bound}, let
\[
    U_\epsilon
    :=
    \left\{
    A\in \mathbb{B}_F^1 :
    \operatorname{dist}\left(A,\mathbb{S}_{B}^{r,\rho}\cap \mathbb{B}_F^2\right)
    \le \epsilon
    \right\}.
\]
Recall that we want to show
\[
    \operatorname{Vol}(U_\epsilon)
    \le
    \frac{1}{2}\operatorname{Vol}(\mathbb{B}_{\mathrm{flat}}^1)
\]
so that some flat matrix lies outside of $U_\epsilon$. It is thus sufficient to show
\[
    \operatorname{Vol}(U_\epsilon)
    \le
    \frac{1}{2}
    \left(\frac{2\sqrt 3}{\sqrt{2\pi e}}\right)^d
    \operatorname{Vol}(\mathbb{B}_F^1).
\]
As in the proof of \cref{thm:fancy lower bound},
\[
    \operatorname{Vol}(U_\epsilon)
    \le
    \mathcal{N}(\epsilon,\|\cdot\|_F,\mathbb{S}_{B}^{r,\rho}\cap \mathbb{B}_F^2)
    (2\epsilon)^{NN'}
    \operatorname{Vol}(\mathbb{B}_F^1).
\]
Thus it suffices to require
\[
    \mathcal{N}(\epsilon,\|\cdot\|_F,\mathbb{S}_{B}^{r,\rho}\cap \mathbb{B}_F^2)
    (2\epsilon)^{NN'}
    \le
    \frac{1}{2}
    \left(\frac{2\sqrt 3}{\sqrt{2\pi e}}\right)^d.
\]
Using the same covering estimates as in \cref{thm:fancy lower bound}, this holds whenever
\[
    \epsilon
    <
    \frac{20}{41}
    \cdot
    \frac{1}
    {(2B+1)^{\frac{J}{J-1}}
    \left(41\rho\frac{\sqrt{2\pi e}}{2\sqrt 3}\right)^{\frac{1}{J-1}}}.
\]
Since $\sqrt{2\pi e}/(2\sqrt 3)=\sqrt{\pi e/6}$, this is exactly the claimed bound.
\end{proof}

\subsection{Proof of Theorem \ref{thm:main_channel}}

\begin{proof}
Let $\hat X=XV_r$, and write
\(
    X_+
    :=
    \begin{pmatrix}
        X\\
        \sqrt{\lambda}I_1
    \end{pmatrix},
    \)
    \(\hat X_+
    :=
    \begin{pmatrix}
        \hat X\\
        \sqrt{\lambda}I_2
    \end{pmatrix}
    \), \(
    \mathbb X_+:=\begin{bmatrix}X_+ & \hat X_+\end{bmatrix}.
\)
By \cref{lemma:err evolution}, applied to the augmented calibration matrix $\mathbb X_+$,
\begin{align*}
\left\|
\begin{pmatrix}
    X\mathbf w\\
    \sqrt{\lambda}\mathbf w\\
    \mathbf 0
\end{pmatrix}
-
\begin{pmatrix}
    X(\mathbf q+L\mathbf r)\\
    \sqrt{\lambda}\mathbf q\\
    \sqrt{\lambda}\mathbf r
\end{pmatrix}
\right\|_2^2 
\quad &= \quad 
\left\|
\mathbb X_+
\begin{bmatrix}
    \mathbf w\\
    \mathbf 0
\end{bmatrix}
-
\mathbb X_+
\begin{bmatrix}
    \mathbf q\\
    \mathbf r
\end{bmatrix}
\right\|_2^2
\\
& =
\sum_{j=1}^N
|w_j^{(j-1)}-q_j|^2
\left\|
P_{((\mathbb X_+)_{\ge j+1})^\perp}(X_+)_j
\right\|_2^2 \\
& \le
\frac{\delta^2}{4}
\sum_{j=1}^N
\left\|
P_{((\mathbb X_+)_{\ge j+1})^\perp}(X_+)_j
\right\|_2^2 \\
& \le
\frac{\delta^2}{4}
\sum_{j=1}^N
\left\|
P_{\hat X_+^\perp}(X_+)_j
\right\|_2^2 \\
& =
\frac{\delta^2}{4}
\|P_{\hat X_+^\perp}X_+\|_F^2
=
\frac{\delta^2}{4}
\left(
\|X_+\|_F^2-\|P_{\hat X_+}X_+\|_F^2
\right).
\end{align*}

\SZnew{The second equality uses the recursive error evolution identity \eqref{final error} only for the first $N$ steps, as \cref{lora_in_GPTQ} terminates at $N$ iterations.} The second inequality uses the fact that, for $j\le N$, the columns of $\hat X_+$ are contained among the future columns of $\mathbb X_+$.
We now estimate the final term. Since $X=U\Sigma V^\top$ and $L=V_r$, we have $\hat X=XV_r=U_r\Sigma_r$. Also,
\[
    \|X_+\|_F^2
    =
    \|\Sigma\|_F^2+\lambda N
    =
    \sum_{i=1}^N \sigma_i^2+\lambda N.
\]
Moreover,
\[
    P_{\hat X_+}
    =
    \hat X_+(\hat X_+^\top\hat X_+)^{-1}\hat X_+^\top
    =
    \begin{pmatrix}
        U_r\Sigma_r\\
        0_{N\times r}\\
        \sqrt{\lambda}I_r
    \end{pmatrix}
    (\Sigma_r^2+\lambda I)^{-1}
    \begin{bmatrix}
        \Sigma_r U_r^\top & 0_{r\times N} & \sqrt{\lambda}I_r
    \end{bmatrix}.
\]
Therefore,
\[
    P_{\hat X_+}X_+
    =
    \begin{pmatrix}
        U_r\Sigma_r(\Sigma_r^2+\lambda I)^{-1}\Sigma_r^2V_r^\top\\
        0_{N\times {N}}\\
        \sqrt{\lambda}(\Sigma_r^2+\lambda I)^{-1}\Sigma_r^2V_r^\top
    \end{pmatrix}.
\]
Since $V_r$ has orthonormal columns,
\begin{align*}
\|P_{\hat X_+}X_+\|_F^2
&=
\left\|
\Sigma_r(\Sigma_r^2+\lambda I)^{-1}\Sigma_r^2
\right\|_F^2
+
\left\|
\sqrt{\lambda}(\Sigma_r^2+\lambda I)^{-1}\Sigma_r^2
\right\|_F^2 \\
&=
\sum_{i=1}^r
\frac{\sigma_i^4}{\sigma_i^2+\lambda}
\ge
\sum_{i=1}^r
\frac{\sigma_i^4-\lambda^2}{\sigma_i^2+\lambda}
=
\sum_{i=1}^r \sigma_i^2-r\lambda.
\end{align*}
Combining these estimates gives
\[
\begin{aligned}
\|X\mathbf w-X(\mathbf q+L\mathbf r)\|^2
+\lambda\|\mathbf w-\mathbf q\|^2
+\lambda\|\mathbf r\|^2 
& \le
\frac{\delta^2}{4}
\left(
\sum_{i=1}^N\sigma_i^2+\lambda N
-
\sum_{i=1}^r\sigma_i^2
+
r\lambda
\right) \\
& =
\frac{\delta^2}{4}
\left(
\|X-X_r\|_F^2+(N+r)\lambda
\right),
\end{aligned}
\]
as claimed.
\end{proof}

\section{Optimal Low-Rank Compensation via GSVD}\label{sec:GSVD}

This appendix recalls the linear algebra underlying the optimal low-rank compensation step used in \cref{OLrC}. The main point is that, once the quantized matrix $Q$ is fixed, the best rank $r$ correction is a weighted low-rank approximation. This can be expressed either through a generalized singular value decomposition or, more generally, as a special case of Penrose regression.

Consider a matrix $A\in\mathbb{R}^{m\times n}$. Its singular value decomposition is $A=U\Sigma V^\top$, where $U$ and $V$ are orthogonal matrices and $\Sigma$ is diagonal, with nonnegative diagonal entries in decreasing order. The columns of $U$ and $V$ are the left and right singular directions of $A$, while the diagonal entries of $\Sigma$ are the singular values.
We will use the SVD through the Eckart-Young theorem \cite{eckart1936approximation}. Namely, for any positive integer $k\le \operatorname{rank}(A)$, the best rank $k$ approximation to $A$, in both Frobenius and operator norm, is $A_k=\sum_{i=1}^k \sigma_i u_i v_i^\top$, where $\sigma_i$ are the singular values of $A$, and $u_i$, $v_i$ are the corresponding left and right singular vectors.

We next recall a generalized singular value decomposition adapted to weighted inner products. Let $M_1$ and $M_2$ be positive definite matrices of sizes $m\times m$ and $n\times n$, respectively. A generalized singular value decomposition of $A$ with respect to $(M_1,M_2)$ is a factorization $A=U\Sigma V^\top$, where $U^\top M_1U=I$, $V^\top M_2V=I$, and $\Sigma$ is diagonal. Thus $U$ and $V$ are orthonormal with respect to the weighted inner products induced by $M_1$ and $M_2$. This decomposition is closely related to constrained principal component analysis and redundancy analysis \cite{takane2001constrained,takane2007regularized,takane2008regularized}.

The GSVD can be computed from an ordinary SVD. If $M_1^{1/2}AM_2^{1/2}=\widetilde U\Sigma\widetilde V^\top$ is a standard SVD, then $U=M_1^{-1/2}\widetilde U$ and $V=M_2^{-1/2}\widetilde V$ give the desired weighted decomposition. We write $\mathcal T_r(A)$ for the usual rank $r$ SVD truncation and $\widetilde{\mathcal T}_r(A)$ for the corresponding rank $r$ GSVD truncation, with the weighting matrices understood from context.
In the case of \cref{OLrC}, which is also the formulation appearing in \cite{liu2024eora,zhang2024qera,deng2025cloq,saha2024compressing}, the relevant weights are $(M_1,M_2)=(X^\top X,I)$. More generally, assuming $X$ has full column rank, the problem $\min_{\operatorname{rank}(M)\le r}\|E-XM\|_F^2$ has solution $M^\star=\widetilde{\mathcal T}_r(X^\dagger E)=(X^\top X)^{-1/2}\mathcal T_r((X^\top X)^{1/2}X^\dagger E)$. In the quantization error compensation setting, $E=X(W-Q)$, so $X^\dagger E=W-Q$. Thus OLrC computes a weighted rank $r$ approximation of $W-Q$, where the weighting is induced by the calibration Hessian $X^\top X$. 

This problem can also be viewed as a special case of Penrose regression \cite{ten1993least,takane2007regularized}.
The more general Penrose regression problem is
\begin{equation}\label{Penrose}
\begin{aligned}
    \min_{Z} \quad & h(Z):=\|AZB^\top-Y\|_F^2, \\
    \text{subject to} \quad & \operatorname{rank}(Z)\le r,
\end{aligned}
\end{equation}
where $Y$ is given and $A$ and $B$ are full column rank matrices, so $A^\top A$ and $B^\top B$ are invertible. Let $\widehat Z=A^\dagger Y(B^\top)^\dagger$ be the unconstrained least squares estimator, and define $G=(A^\top A)^{1/2}\widehat Z(B^\top B)^{1/2}$. The key identity is
\begin{equation}\label{Identity}
\begin{aligned}
    h(Z)
    &=
    \|A\widehat ZB^\top-Y\|_F^2
    +
    \operatorname{tr}\left(
    B(\widehat Z-Z)^\top A^\top A(\widehat Z-Z)B^\top
    \right) \\
    &=
    \left\|G-(A^\top A)^{1/2}Z(B^\top B)^{1/2}\right\|_F^2
    +
    \operatorname{tr}(Y^\top Y)
    -
    \operatorname{tr}(G^\top G).
\end{aligned}
\end{equation}
Consequently, minimizing \eqref{Penrose} is equivalent to computing a best rank $r$ approximation of $G$. Hence the minimizer is obtained by the rank $r$ GSVD truncation of $\widehat Z$ with respect to $(A^\top A,B^\top B)$, namely $Z^\star=\widetilde{\mathcal T}_r(\widehat Z)$.

This general solution also appears in low-rank compression of neural networks. For example, the Optimal Brain Decomposition method of \citet{li2026optimal} considers $\min_{\operatorname{rank}(M)\le r}\|L_x^\top(W-M)L_g\|_F^2$, where $L_x$ and $L_g$ are Cholesky factors of the input Hessian $H_x$ and output-gradient Hessian $H_g$, respectively. The proposed solution $(L_x^\top)^{-1}\mathcal T_r(L_x^\top W L_g)L_g^{-1}$ is precisely a truncated GSVD of $W$ with respect to $(H_x,H_g)$, where $L_x^\top$ and $L_g$ play the roles of $H_x^{1/2}$ and $H_g^{1/2}$.

Finally, the truncated SVD and GSVD can often be computed efficiently using randomized numerical linear algebra. Dense SVD or eigendecomposition costs $\mathcal O(N^3)$ for an $N\times N$ matrix, which can be expensive when $N$ is large. In the low-rank regime, randomized Krylov methods such as RBKI \cite{tropp2023randomized} can compute truncated approximations using roughly $\mathcal O(N^2r)$ operations for dense matrices, while column Nyström methods such as RPCholesky \cite{chen2022randomly} can have cost on the order of $\mathcal O(Nr^2)$, depending on the access model. Fast least squares solvers tailored to machine learning optimization provide another related approach \cite{maalouf2019fast,maalouf2022fast}.

\section{Implementation Details: Dampening and Stable Triangular Factorization}\label{sec:better cholesky}

This section describes some numerical choices used in our implementation. First, we explain how dampening is applied consistently across OLrC and Bid-Up during refinement. Second, we describe a stable way to compute the triangular factor needed by GPTQ and GPTQ-intrinsic LoRA. This second point is important because the augmented Hessian used in GPTQ-intrinsic LoRA is singular before regularization and can remain highly ill-conditioned after standard dampening. The goal of the procedure below is not to change the underlying GPTQ update, but to compute the same inverse-Hessian factor more robustly, so that the regularization parameter $\lambda$ can be chosen for the quantization objective rather than for the numerical success of a Cholesky decomposition.

\paragraph{Dampening for OLrC and Bid-Up.} When implementing OLrC (\cref{OLrC}), we use the same dampened Hessian as in GPTQ. That is, we replace $H=X^\top X$ by $H+\lambda I$ when computing the matrix square root and inverse square root needed in the GSVD formulation. This improves numerical stability when computing $(H+\lambda I)^{1/2}$ and $(H+\lambda I)^{-1/2}$ by eigen-decomposition. We also remark that computing the full SVD (using \textit{torch.svd}) before truncation to rank $r$ is not only computationally unnecessary, but may even fail due to the SVD routine not converging. We thus use \textit{torch.svd\_lowrank} (based on a randomized procedure in \cite{halko2011finding}) to directly compute the best low rank approximation with a slightly enlarged target rank, say $r+10$, then truncate to the target rank $r$. 

Bid-Up (\cref{Bid-up}) only requires the reciprocal of the diagonal entries of the Hessian, and therefore does not have the same numerical stability issue associated with matrix square roots or inverse square roots. Nevertheless, in our refinement loops we also replace $H$ by $H+\lambda I$ in Bid-Up. The reason is that the OLrC and Bid-Up steps are meant to alternate on the same objective. With $Q$ fixed, applying dampening in OLrC is equivalent to minimizing the regularized loss $\|X(W-Q-LR)\|_F^2+\lambda\|W-Q-LR\|_F^2$ over $LR$. Applying the same dampening in Bid-Up ensures that the quantized update is also performed with respect to this regularized objective.

\paragraph{Dampening for GPTQ-intrinsic LoRA.}
We follow the standard choice of setting $\lambda=0.01\,\operatorname{mean}(\operatorname{diag}(\mathbb{H}))$ to form the dampened augmented Hessian $\mathbb{H}+\lambda I$ for all layers in all language and vision models we tested, with the following exception. We observe that the inputs to the down-projection layer in the third transformer block of the Qwen3-0.6B and 1.7B models exhibit a highly concentrated singular value spectrum, in which a single leading singular value overwhelmingly dominates all remaining singular values and attains a very large magnitude. Only when forming the dampened augmented Hessian for that specific layer in these two models, we choose a slightly bigger $\lambda=0.05\,\operatorname{mean}(\operatorname{diag}(\mathbb{H}))$. In the 3-bit results with 1 refinement loop, we also skip this special layer and only perform the refinement loop for the other 195 linear layers in these two models.

\paragraph{Stable Triangular Factorization.} As discussed in \cref{sec:method}, the augmented Hessian $\mathbb{H}$ in GPTQ-intrinsic LoRA (\cref{lora_in_GPTQ}) is more ill-conditioned than the original Hessian $H$. In fact, before regularization, $\mathbb{H}$ is singular because the augmented features include $\hat X=XL$, whose columns lie in the column span of $X$. We observe that, with the standard GPTQ choice $\lambda=0.01\,\operatorname{mean}(\operatorname{diag}(\mathbb{H}))$, the original Cholesky-based GPTQ implementation can fail when applied to $\mathbb{H}+\lambda I$, since this matrix may still be too ill-conditioned. To avoid increasing $\lambda$ solely for numerical reasons, we use a more robust procedure for computing the triangular factor required by GPTQ.

Specifically, we first compute an eigen-decomposition $\mathbb{H}+\lambda I=PSP^\top$. Then we form the inverse square root $(\mathbb{H}+\lambda I)^{-1/2}=PS^{-1/2}P^\top$. Next, we compute a QR decomposition $(\mathbb{H}+\lambda I)^{-1/2}=OG$, where $O$ is orthogonal and $G$ is upper triangular. We then set the lower triangular GPTQ factor to be $\Psi_{\mathbb H}=G^\top$. Indeed,
\[
    \Psi_{\mathbb H}\Psi_{\mathbb H}^\top
    =
    G^\top G
    =
    \bigl((\mathbb{H}+\lambda I)^{-1/2}\bigr)^\top
    (\mathbb{H}+\lambda I)^{-1/2}
    =
    (\mathbb{H}+\lambda I)^{-1}.
\]
With the standard QR sign convention in which $G$ has positive diagonal entries, $\Psi_{\mathbb H}$ is the Cholesky factor of $(\mathbb{H}+\lambda I)^{-1}$. The key practical point is that this computes the factor needed by GPTQ without applying Cholesky decomposition directly to an ill-conditioned inverse Hessian.

This distinction matters because $\lambda$ plays several roles. With an infinite quantization grid, smaller $\lambda$ is preferable from the point of view of reconstruction error, provided $(H+\lambda I)^{-1}$ can be computed stably. With a finite grid, $\lambda$ must also be large enough to help keep the adaptive GPTQ updates within a mildly enlarged quantization range. In a Cholesky-based implementation, however, $\lambda$ may need to be chosen even larger merely so that the Cholesky decomposition of $(H+\lambda I)^{-1}$ succeeds. This last requirement is an implementation artifact rather than a feature of the quantization objective.

Thus, for ill-conditioned or low-rank Hessians, directly computing the Cholesky decomposition of $(H+\lambda I)^{-1}$ can force an unnecessarily large choice of $\lambda$. By computing the triangular factor through eigen-decomposition followed by QR, $\lambda$ can instead be tuned according to the quantization tradeoff of balancing reconstruction distortion against the boundedness requirements imposed by the finite grid.

\end{document}